\documentclass[11pt]{article}

\usepackage[preprint]{acl}

\usepackage{times}
\usepackage{latexsym}

\usepackage[T1]{fontenc}
\usepackage[utf8]{inputenc}
\usepackage{microtype}
\usepackage{inconsolata}
\usepackage{graphicx}
\usepackage{booktabs}
\usepackage{amsfonts}
\usepackage{amsmath}
\usepackage{amssymb}
\usepackage{nicefrac}
\usepackage{amsthm}
\usepackage{enumitem}
\usepackage{subcaption}

\newtheorem{proposition}{Proposition}
\newtheorem{lemma}[proposition]{Lemma}

\newcommand{\Var}{\operatorname{Var}}
\newcommand{\E}{\operatorname{E}}
\newcommand{\Cov}{\operatorname{Cov}}
\newcommand{\MS}{\mathrm{MS}}
\newcommand{\Ktot}{K_{\text{tot}}}

\title{CyclicJudge: Mitigating Judge Bias Efficiently in LLM-based Evaluation}

\author{%
  \textbf{Ziyi Zhu\textsuperscript{1}, Olivier Tieleman\textsuperscript{1}, Alexey Bukhtiyarov\textsuperscript{1}, Jinghong Chen\textsuperscript{2}}\\
  \textsuperscript{1}Slingshot AI \quad \textsuperscript{2}Department of Engineering, University of Cambridge\\
  \texttt{\{ziyi,olivier\}@slingshotai.com} \quad \texttt{jc2124@cam.ac.uk}
}

\begin{document}

\maketitle

\begin{abstract}
LLM-as-judge evaluation has become standard practice for open-ended model assessment; however, judges exhibit systematic biases that cannot be averaged out by increasing the number of scenarios or generations. These biases are often similar in magnitude to the model differences that benchmarks are designed to detect, resulting in unreliable rankings when single-judge evaluations are used. We introduce a variance decomposition that partitions benchmark score variance into scenario, generation, judge, and residual components. Based on this analysis, \textbf{CyclicJudge}, a round-robin assignment of judges to scenarios, is demonstrated to be the optimal strategy for a fixed judge panel and judge-call budget: it cancels the average additive judge offsets while matching the cost of single-judge evaluation. Empirical results on MT-Bench and MindEval validate the effectiveness of CyclicJudge as predicted, across both general-purpose and domain-specific evaluation settings, reducing benchmark score variance and stabilizing the resulting model ranking.
\end{abstract}

\section{Introduction}

LLM-as-judge evaluation has become the de facto standard for open-ended model assessment \citep{zheng2023judging}. A growing body of work indicates that LLM judges are systematically biased: some models are consistently lenient, others strict, with additional position, verbosity, and self-preference effects \citep{zheng2023judging, shi2024judging, wang2023large}. A standard Analysis of Variance (ANOVA; \citealp{fisher1925statistical}) on our data confirms that judge main effect is highly significant (\S\ref{sec:anova}).

Crucially, systematic bias does not average away across more scenarios or generations: reducing them requires aggregating across multiple judges. Yet running every judge on every item multiplies evaluation cost by the panel size and sacrifices generation diversity at a fixed budget. Throughout this work, we treat the mean score of a diverse $\Ktot$-judge panel as the reference for evaluation; our goal is to recover this panel mean efficiently.

We show that \textbf{CyclicJudge}, a round-robin assignment of judges to items, resolves this tension. To establish this formally, we develop a variance decomposition rooted in generalizability theory \citep{brennan2001generalizability, shavelson1991generalizability} that separates benchmark score variance into scenario, generation, judge, and residual components. Our contributions are:
\begin{enumerate}[leftmargin=*,itemsep=2pt]
    \item A mixed-effects model that decomposes benchmark variance into random noise and panel-relative judge offsets, showing they require fundamentally different remedies (\S\ref{sec:model}--\ref{sec:variance}).
    \item A proof that round-robin cycling achieves lower variance than alternative strategies: it cancels the average panel-relative offsets while maximizing generation diversity at any budget (\S\ref{sec:allocation}).
    \item Empirical results on MT-Bench and MindEval showing that CyclicJudge lowers benchmark score variance as predicted (\S\ref{sec:strategy_comparison}) and recovers the panel leaderboard more reliably than the alternative strategies at equal budget (\S\ref{sec:ranking}).
\end{enumerate}

\section{Related Work}

\paragraph{LLM judge biases and multi-judge approaches.} LLM judges exhibit systematic position bias \citep{wang2023large, shi2024judging}, length bias \citep{dubois2024length}, self-preference \citep{wataoka2024self}, and cognitive biases \citep{koo2023benchmarking}; we refer readers to \citet{gu2024survey} for a comprehensive survey. Multi-judge panels mitigate these biases by averaging over diverse judge tendencies \citep{verga2024replacing}, while psychometric reliability analyses further demonstrate that single-judge evaluations yield insufficiently stable scores \citep{schroeder2024trust}.

\paragraph{Benchmark reliability and generation variance.} \citet{madaan2024quantifying} empirically quantifies variance across hundreds of models and benchmarks, and \citet{alvarado2025repetitions} shows via ANOVA that unrepeated evaluations frequently produce rank inversions across benchmark subcategories. \citet{zhang2025beyond} proposes a hierarchical model decomposing variance into within-prompt and between-prompt components, and \citet{mizrahi2024state} shows that prompt paraphrasing alone can reverse model rankings. Our work extends this line by further separating systematic judge bias from random noise.

\paragraph{Budget-constrained evaluation.} \citet{perlitz2023efficient} frames benchmark design as a compute--reliability tradeoff and shows that most of a benchmark's reliability survives evaluating only a fraction of its examples; our allocation analysis asks the same question along the judge axis instead of the example axis. \citet{saha2026budget} frames judge selection as a multi-armed bandit problem, \citet{luettgau2025hibayes} applies hierarchical Bayesian partial pooling in low-data regimes, and \citet{miller2024error} derives standard errors and power formulas for benchmark scores; \citet{dorner2024limits} further establishes fundamental scaling limits on the accuracy gains from additional judge queries.

\section{Methodology}

\subsection{Model Specification}
\label{sec:model}

Consider evaluating model $\theta$ on $n$ scenarios from a dataset $\mathcal{D}$ of $N$ total scenarios, with $m$ independent generations per scenario, each scored by $K$ judges drawn from a pool of $\Ktot$. The same $K$ judges score every scenario--generation cell (a crossed design). We model each score as:
\begin{equation}
\label{eq:model}
X_{ij\ell} = \mu_\theta + \alpha_i + \beta_{ij} + \gamma_\ell + \varepsilon_{ij\ell}
\end{equation}
where $i \in \{1,\ldots,n\}$ indexes scenarios, $j \in \{1,\ldots,m\}$ indexes generations within scenario $i$, and $\ell$ indexes the $K$ selected judges. Each term has a distinct source:
\begin{itemize}[leftmargin=*,itemsep=2pt]
\item $\mu_\theta$: \emph{panel-mean estimand} for model $\theta$---the average score the $\Ktot$-judge panel would assign in expectation, with random noise removed.
\item $\alpha_i \overset{\text{iid}}{\sim} (0,\,\sigma_\alpha^2)$: scenario effect---how much individual scenario difficulty varies for model $\theta$. A dataset whose scenarios all test the same aspects of model $\theta$ equally has $\sigma_\alpha^2 = 0$.
\item $\beta_{ij} \overset{\text{iid}}{\sim} (0,\,\sigma_\beta^2)$: generation effect---variability from stochastic decoding. $\sigma_\beta^2$ is model-specific; a deterministic model has $\sigma_\beta^2 = 0$.
\item $\gamma_\ell$: \emph{judge offset relative to the panel mean}---a fixed constant for each judge, with centering constraint $\sum_{\ell=1}^{\Ktot}\gamma_\ell = 0$ and population dispersion $\sigma_\gamma^2 = \Ktot^{-1}\sum_\ell \gamma_\ell^2$. $\gamma_\ell$ measures deviation from the panel reference, not from absolute truth: bias shared by all panel members is absorbed into $\mu_\theta$.
\item $\varepsilon_{ij\ell} \overset{\text{iid}}{\sim} (0,\,\sigma_\varepsilon^2)$: residual---captures judge-level noise and all interaction terms; we discuss these modeling choices in Appendix~\ref{app:interactions}.
\end{itemize}
The random effects $(\alpha, \beta, \varepsilon)$ are mutually independent with finite second moments; no distributional assumptions (e.g.,\ Gaussianity) are needed. The $n$ scenarios and $K$ judges are each drawn by simple random sampling without replacement, independently of each other and of all random effects.

\subsection{Variance Decomposition}
\label{sec:variance}

The benchmark score is the grand mean $\bar{X} = \frac{1}{nmK}\sum_{i,j,\ell} X_{ij\ell}$, which decomposes as $\bar{X} = \mu_\theta + \bar{\alpha} + \bar{\beta} + \bar{\gamma} + \bar{\varepsilon}$, where each bar denotes a hierarchical sample mean.

\begin{proposition}[Variance decomposition]
\label{prop:variance}
\begin{equation}
\label{eq:variance}
\Var(\bar{X}) = \underbrace{\frac{\sigma_\alpha^2}{n} + \frac{\sigma_\beta^2}{nm} + \frac{\sigma_\varepsilon^2}{nmK}}_{\text{random noise}} + \underbrace{\frac{\sigma_\gamma^2}{K}\cdot\frac{\Ktot - K}{\Ktot - 1}}_{V_\gamma\text{: judge bias}}
\end{equation}
\end{proposition}
Each noise term shrinks with more data at or above its hierarchy level. The bias term $V_\gamma$ responds only to $K$ and vanishes exactly at $K = \Ktot$. The proof uses the law of total variance, conditioning on the judge selection, combined with a finite population correction for sampling $K$ from $\Ktot$ centered bias values (Appendix~\ref{app:derivation}).

\subsection{Allocation Strategies}
\label{sec:allocation}

Given a per-scenario budget of $B$ judge calls, the design question is how to allocate judges across generations. The scenario term $\sigma_\alpha^2/n$ is fixed once scenarios are chosen; we compare three strategies for the remaining variance.

\paragraph{Strategy A: All judges per generation.} Use all $\Ktot$ judges for each of $m = B/\Ktot$ generations. The mean of all judge offsets is zero ($\bar{\gamma} = 0$):
\begin{equation}
\label{eq:strat_all}
V_A(B) = \frac{\Ktot\,\sigma_\beta^2 + \sigma_\varepsilon^2}{nB}
\end{equation}

\paragraph{Strategy B: Random single judge per generation.} Use $m = B$ generations, each scored by one randomly sampled judge. Random judge selection injects the panel-relative offset as additional sampling noise:
\begin{equation}
\label{eq:strat_rand}
V_B(B) = \frac{\sigma_\beta^2 + \sigma_\gamma^2 + \sigma_\varepsilon^2}{nB}
\end{equation}

\paragraph{Strategy C: CyclicJudge (round-robin).} Use $m = B$ generations, assigning judge $j \bmod \Ktot$ to generation $j$. Each judge appears equally often, so the panel-relative offsets sum to zero ($\bar{\gamma} = 0$):
\begin{equation}
\label{eq:strat_cycle}
V_C(B) = \frac{\sigma_\beta^2 + \sigma_\varepsilon^2}{nB}
\end{equation}

\paragraph{Cycling across scenarios.} When $B$ is not a multiple of $\Ktot$, we can write $B = q\Ktot + r$ with $0 \le r < \Ktot$. Each scenario receives $q$ full cycles, which cancel within the scenario, plus $r$ remainder calls. We run one continuous cycle across the whole benchmark: shuffle the scenarios, then hand out judges in round-robin order to the $nB$ judge calls, so each scenario's remainder picks up where the previous one left off. Shuffling keeps the cycle from aligning with any latent ordering of the dataset. Offsets then cancel \emph{exactly} whenever $\Ktot \mid nB$, and otherwise leave fewer than $\Ktot$ uncancelled calls, a residual bias adding $O\bigl((1/nB)^2\bigr)$ to the mean squared error (Appendix~\ref{app:allocation}). In the case of the smallest budget $B = 1$, where $q = 0$ and $r = 1$, every call is a remainder call, so each scenario receives a single generation scored by a single judge, and the cycle runs entirely across scenarios.

\paragraph{Advantage of CyclicJudge.} Direct comparison (Appendix~\ref{app:allocation}) shows $V_C \leq \min(V_A, V_B)$ for all $B$, with gaps of $O(1/B)$, so CyclicJudge's advantage is largest at small budgets of judge calls. The ranking of A vs.\ B depends on whether $(\Ktot{-}1)\sigma_\beta^2 \gtrless \sigma_\gamma^2$: all judges beat random only when judge bias is large relative to generation variance.

\section{Experiments}
\label{sec:experiments}

\subsection{Domain and Data}

We validate our approach on two benchmarks:

\paragraph{MT-Bench} \citep{zheng2023judging} is a general-purpose multi-turn conversational benchmark with 80 two-turn questions spanning 8 categories, scored on a 1--10 scale per turn (averaged).

\paragraph{MindEval} \citep{pombal2025mindeval} is a domain-specific benchmark for multi-turn mental health support with 50 scenarios and composite scores on a 1--5 scale covering empathy, safety, and clinical appropriateness. Each scenario simulates a full 10-turn therapeutic session with a synthetic patient.

For both benchmarks, we evaluate five models---Qwen 2.5 7B Instruct \citep{yang2024qwen25}, Llama 3.3 70B Instruct \citep{grattafiori2024llama}, GPT-5.2, Gemini 3 Flash, and Claude Sonnet 4.6, which also serve as judges in a symmetric design. MT-Bench uses $m{=}10$ generations per scenario (4{,}000 scores per model); MindEval uses $m{=}5$ (1{,}250 scores per model). We follow the default judge decoding hyperparameters from reference (Appendix~\ref{app:hyperparameters}).

\subsection{Judge Bias Estimates}
\label{sec:anova}

As a complementary omnibus check on the mixed-effects decomposition (\S\ref{sec:model}), we test whether judge identity has a significant effect using a two-way ANOVA, treating each scenario--generation pair as a subject crossed with $K$ judges. On both benchmarks, the judge main effect is highly significant ($p < 0.001$ in all cases; see Appendix~\ref{app:anova}).

Table~\ref{tab:mean_scores} demonstrates that \textbf{single-judge evaluation produces unreliable rankings}. On MT-Bench at the default operating point ($m{=}1$, $K{=}1$), the top-ranked model varies with the judge: three different models take the top slot across the five judges, and two of the five disagree with the panel winner (Appendix~\ref{app:ranking}). Qwen ties itself for first under its own judgment despite ranking last by consensus, a pattern consistent with the self-preference bias documented by \citet{wataoka2024self}. The per-observation variance estimates (Table~\ref{tab:per_obs}) show that judge variance $\hat{\sigma}^2_\gamma$ ranges from 0.34 to 0.95 across models. Since competitive models on current leaderboards are routinely separated by $< 0.5$ points \citep{chiang2024chatbot}, single-judge evaluation can easily reverse rankings. On MindEval, rankings are also notably more stable across single judges: four of five produce an identical ordering. Two plausible factors include: domain-specific rubrics may leave less room for stylistic self-preference, and the markedly lower scenario and generation variance on MindEval (\S\ref{sec:variance_components}) compresses the absolute size of any model-dependent fluctuations.

\begin{table}[t]
\centering
\caption{Mean score $\pm$ 95\% CI at the default operating point ($m{=}1$, $K{=}1$). \textbf{Bold}: highest score per judge column. \underline{Underline}: second highest.}
\label{tab:mean_scores}
\smallskip
{\small (a) MT-Bench (1--10 scale, $n{=}80$)}
\vspace{2pt}

\resizebox{\columnwidth}{!}{%
\begin{tabular}{@{}lccccc@{}}
\toprule
& \multicolumn{5}{c}{\textbf{Judge}} \\
\cmidrule(l){2-6}
\textbf{Model} & Qwen & Llama & GPT & Gemini & Claude \\
\midrule
Gemini & \textbf{8.02$\pm$.16} & \underline{8.98$\pm$.08} & \underline{8.20$\pm$.12} & \textbf{9.88$\pm$.06} & \textbf{8.68$\pm$.10} \\
GPT    & 7.91$\pm$.16 & 8.81$\pm$.11 & \textbf{8.67$\pm$.12} & \underline{9.71$\pm$.11} & 8.49$\pm$.13 \\
Claude & 7.96$\pm$.17 & \textbf{9.06$\pm$.06} & 7.91$\pm$.13 & 9.62$\pm$.12 & \underline{8.52$\pm$.13} \\
Llama  & 7.87$\pm$.18 & 8.71$\pm$.11 & 6.90$\pm$.20 & 8.58$\pm$.25 & 7.26$\pm$.22 \\
Qwen   & \textbf{8.02$\pm$.15} & 8.39$\pm$.16 & 6.09$\pm$.20 & 7.16$\pm$.27 & 5.99$\pm$.23 \\
\bottomrule
\end{tabular}}

\medskip
{\small (b) MindEval (1--5 scale, $n{=}50$)}
\vspace{2pt}

\resizebox{\columnwidth}{!}{%
\begin{tabular}{@{}lccccc@{}}
\toprule
& \multicolumn{5}{c}{\textbf{Judge}} \\
\cmidrule(l){2-6}
\textbf{Model} & Qwen & Llama & GPT & Gemini & Claude \\
\midrule
Claude  & 3.86$\pm$.08 & \textbf{4.67$\pm$.07} & \textbf{4.05$\pm$.06} & \textbf{4.72$\pm$.06} & \textbf{4.84$\pm$.04} \\
GPT     & \textbf{4.04$\pm$.09} & \underline{4.51$\pm$.07} & \underline{3.76$\pm$.09} & \underline{3.95$\pm$.16} & \underline{3.39$\pm$.15} \\
Gemini  & \underline{3.90$\pm$.08} & 4.49$\pm$.06 & 3.63$\pm$.05 & 3.91$\pm$.12 & 3.31$\pm$.08 \\
Llama   & 3.70$\pm$.07 & 4.29$\pm$.05 & 3.16$\pm$.06 & 3.00$\pm$.10 & 2.85$\pm$.06 \\
Qwen    & 3.26$\pm$.12 & 3.38$\pm$.21 & 2.32$\pm$.08 & 1.69$\pm$.08 & 1.74$\pm$.05 \\
\bottomrule
\end{tabular}}
\end{table}

\subsection{Variance Components}
\label{sec:variance_components}

Table~\ref{tab:per_obs} reports the per-observation variance component estimates via crossed ANOVA (see Appendix~\ref{app:estimation}). On MT-Bench, scenario variance $\hat{\sigma}^2_\alpha$ is the largest per-observation component; maximizing the number of distinct scenarios always reduces total variance (Appendix~\ref{app:scenario_coverage}). Generation variance $\hat{\sigma}^2_\beta$ drops with capability, consistent with the observation that \textbf{stronger models produce more deterministic outputs}. As frontier models increasingly rely on RL post-training, generation variance may shrink further \citep{matsutani2025rl}. The judge offset variance $\hat{\sigma}^2_\gamma$ remains substantial for all models, accounting for $>$94\% of benchmark variance at the default operating point (see Appendix~\ref{app:variance_tables}).

On MindEval, both scenario and generation variance are markedly lower. Scenario variance $\hat{\sigma}^2_\alpha$ drops by one to two orders of magnitude because all MindEval scenarios evaluate the same rubrics across different user profiles, so difficulty varies far less than on MT-Bench. Generation variance is similarly suppressed: each MindEval generation is a full 10-turn therapeutic session, so per-response decoding randomness averages out over many turns. The judge offset variance remains the dominant component throughout, exceeding both scenario and generation variance by an order of magnitude.

\begin{table}[t]
\centering
\caption{Per-observation variance component estimates.}
\label{tab:per_obs}
\smallskip
{\small (a) MT-Bench ($n{=}80$)}
\vspace{2pt}

\resizebox{\columnwidth}{!}{%
\begin{tabular}{@{}lccccc@{}}
\toprule
\textbf{Component} & \textbf{Qwen} & \textbf{Llama} & \textbf{GPT} & \textbf{Gemini} & \textbf{Claude} \\
\midrule
$\hat{\sigma}^2_\alpha$ (scenario)     & 1.530 & 0.882 & 0.634 & 0.408 & 0.393 \\
$\hat{\sigma}^2_\beta$ (generation)  & 0.266 & 0.238 & 0.076 & 0.000 & 0.000 \\
$\hat{\sigma}^2_\gamma$ (judge)    & 0.947 & 0.503 & 0.339 & 0.435 & 0.427 \\
$\hat{\sigma}^2_\varepsilon$ (residual) & 1.486 & 1.130 & 0.564 & 0.661 & 0.830 \\
\bottomrule
\end{tabular}}

\medskip
{\small (b) MindEval ($n{=}50$)}
\vspace{2pt}

\resizebox{\columnwidth}{!}{%
\begin{tabular}{@{}lccccc@{}}
\toprule
\textbf{Component} & \textbf{Qwen} & \textbf{Llama} & \textbf{GPT} & \textbf{Gemini} & \textbf{Claude} \\
\midrule
$\hat{\sigma}^2_\alpha$ (scenario)     & 0.034 & 0.014 & 0.021 & 0.015 & 0.005 \\
$\hat{\sigma}^2_\beta$ (generation)  & 0.100 & 0.019 & 0.047 & 0.016 & 0.002 \\
$\hat{\sigma}^2_\gamma$ (judge)    & 0.523 & 0.280 & 0.132 & 0.150 & 0.155 \\
$\hat{\sigma}^2_\varepsilon$ (residual) & 0.159 & 0.069 & 0.207 & 0.103 & 0.061 \\
\bottomrule
\end{tabular}}
\end{table}

\subsection{Strategy Comparison}
\label{sec:strategy_comparison}

Figure~\ref{fig:strategy} validates the three allocation strategies via 5{,}000 subsampling repetitions at each budget level on both benchmarks. Markers show empirical variance; dashed lines show exact predictions from empirical pool variances (Appendix~\ref{app:pool_pred}). CyclicJudge achieves lower variance everywhere on both benchmarks, and predictions match empirical results precisely. At $B{=}5$, switching from random to cycling cuts variance by ${\sim}$27--40\% across models on MT-Bench.

The optimal strategy among the two alternatives depends on the sign of $(\Ktot{-}1)\sigma_\beta^2 - \sigma_\gamma^2$: all-judges wins when $\sigma_\gamma^2/\sigma_\beta^2 > \Ktot{-}1 = 4$. On MT-Bench, this ratio is below the threshold for Qwen (3.6) and Llama (2.1), so random judging is preferable; for Gemini and Claude ($\hat{\sigma}^2_\beta \approx 0$), all-judges converges to CyclicJudge. On MindEval, the ratio exceeds the threshold for nearly every model, so all-judges consistently outperforms random. The driver is not a larger judge offset variance, but rather a sharply lower generation variance $\hat{\sigma}^2_\beta$. Meanwhile, judges' disagreement over subjective clinical criteria keeps $\hat{\sigma}^2_\gamma$ substantial relative to $\hat{\sigma}^2_\beta$. \textbf{CyclicJudge outperforms both alternatives regardless of this ratio}.

\begin{figure}[t]
\centering
\begin{subfigure}[b]{\linewidth}
\centering
\includegraphics[width=\linewidth]{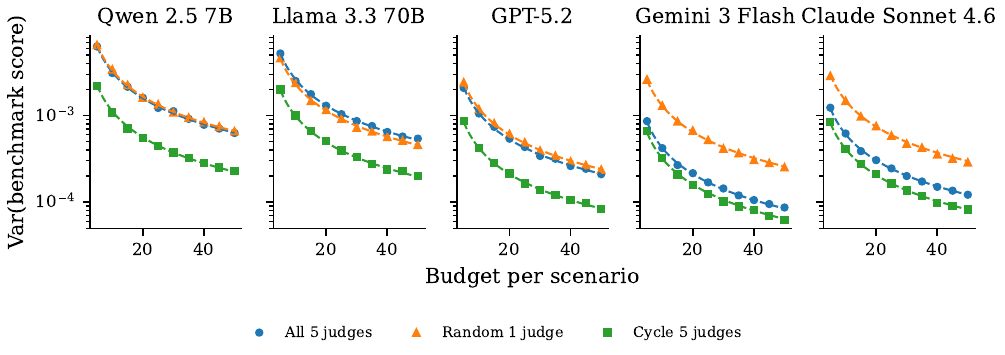}
\caption{MT-Bench (1--10 scale, $n{=}80$, $m_{\max}{=}10$)}
\end{subfigure}
\vspace{0.3em}
\begin{subfigure}[b]{\linewidth}
\centering
\includegraphics[width=\linewidth]{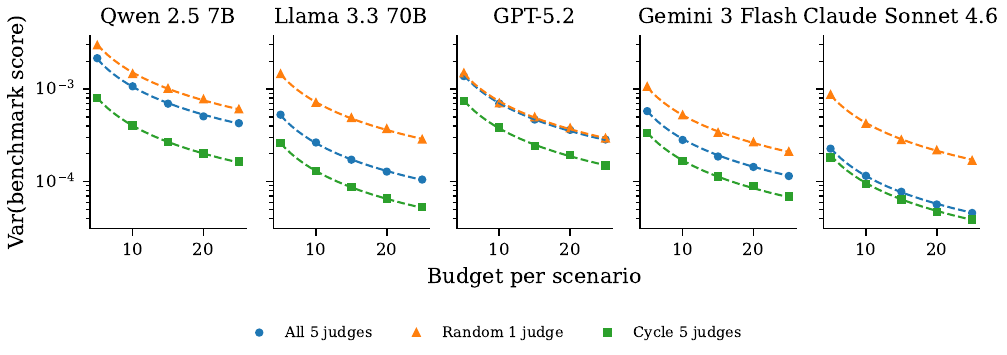}
\caption{MindEval (1--5 scale, $n{=}50$, $m_{\max}{=}5$)}
\end{subfigure}
\caption{Benchmark score variance $\Var(\bar{X})$ vs.\ per-scenario budget for three strategies: all judges (A), random single judge (B), and CyclicJudge (C).}
\label{fig:strategy}
\end{figure}

\subsection{Ranking Stability}
\label{sec:ranking}

Benchmarks are commonly read as leaderboards, so we also measure how often each strategy recovers the \emph{ranking} induced by the reference panel. Alongside the three allocation strategies of \S\ref{sec:allocation} we include a \emph{single judge} baseline: one judge, drawn at random per repetition, scores all $B$ generations of every scenario. The reference is the panel leaderboard computed from the full crossed data using all $\Ktot$ judges and all generations. For each strategy and budget we draw 5{,}000 subsampling repetitions, rank the five models within each repetition, and report the mean Spearman $\rho$ and Kendall $\tau$-b against the reference ranking together with the top-1 flip frequency (see Appendix~\ref{app:ranking}).

Table~\ref{tab:ranking} compares the allocations at $B{=}5$, the smallest budget that gives every scenario a full cycle. A single judge flips the top-ranked model in more than half of the repetitions; \textbf{CyclicJudge reduces the flip rate by 61\% and attains the highest rank correlation on both coefficients}. The ordering of the strategies is preserved as $B$ grows (Appendix~\ref{app:ranking}). MindEval separates the same two regimes: a single judge flips the top model in 20\% of repetitions ($\tau = 0.91$), while all-judges, random, and CyclicJudge all reach $\tau \approx 0.99$ with zero flips at $B{=}5$. MindEval's much smaller scenario and generation variance (\S\ref{sec:variance_components}) leaves the multi-judge strategies little room to differ.

\begin{table}[t]
\centering
\caption{Ranking stability against the reference panel leaderboard on MT-Bench at a common budget of $B{=}5$ judge calls per scenario, over 5{,}000 subsampling repetitions. Higher is better for $\rho$ and $\tau$-b, lower is better for top-1 flip. \textbf{Bold}: best per column.}
\label{tab:ranking}
\smallskip
\resizebox{\columnwidth}{!}{%
\begin{tabular}{@{}lccc@{}}
\toprule
\textbf{Allocation} & \textbf{Spearman} $\rho$ & \textbf{Kendall} $\tau$-b & \textbf{Top-1 flip} \\
\midrule
Single judge    & 0.716 & 0.661 & 0.513 \\
Random 1 judge  & 0.955 & 0.916 & 0.328 \\
All 5 judges    & 0.969 & 0.938 & 0.272 \\
CyclicJudge     & \textbf{0.980} & \textbf{0.960} & \textbf{0.200} \\
\bottomrule
\end{tabular}}
\end{table}

\section{Conclusion}

We develop a variance decomposition for LLM-as-judge evaluation that separates systematic between-judge offsets from random noise and use it to derive CyclicJudge, a round-robin judge assignment strategy. Our analysis and experiments on both a general-purpose benchmark (MT-Bench) and a domain-specific clinical benchmark (MindEval) yield three main findings: (1) Systematic judge offsets dominate benchmark score variance at default operating points. (2) Random noise components decrease with model capability and multi-turn evaluation, but the panel-relative offset variance does not. (3) CyclicJudge achieves the lowest variance among all allocation strategies at every budget level, regardless of the model's variance profile, by canceling the average judge offsets while maximizing generation diversity, and it converts this into a more stable leaderboard: at equal budget it significantly cuts the rate at which the top-ranked model differs from the panel's. The method requires no model-specific tuning and maintains the same per-item cost as single-judge evaluation, providing practitioners with a cost-neutral drop-in replacement for more reliable LLM evaluation.

\section*{Limitations}

\paragraph{Heterogeneous judge cost.}
Our budget $B$ counts judge calls, which is exact only when the panel is price-homogeneous. Under heterogeneous prices, latency, or rate limits the balanced cycle is no longer cost-optimal. Allocation then becomes a weighted design problem in which the optimum trades a bounded offset against reduced noise, which we leave to future work.

\paragraph{Scope of the cancellation guarantee.}
Eq.~\ref{eq:model} treats each judge's bias as a single additive offset $\gamma_\ell$, so what a balanced cycle cancels is each judge's \emph{average} panel-relative offset. Real judges also vary by scenario and by generation, and those interactions are absorbed into the residual (Appendix~\ref{app:interactions}).

\paragraph{Judge panel size.}
Our strategy comparison subsamples from fully crossed evaluations, making an exact ground-truth panel mean available. The cost of full crossing, however, grows as $\Ktot \times (\text{models} \times \text{scenarios} \times \text{repetitions})$, which limits us to $\Ktot = 5$. Empirical validation with larger and more diverse judge pools would strengthen the practical recommendations.

\paragraph{Panel-mean reference.}
Our analysis treats the mean score of the $\Ktot$-judge panel as the reference for evaluation. A diverse panel drawn from disjoint model families typically correlates more strongly with human judgments than any single judge \citep{verga2024replacing}, but the panel mean is not a substitute for absolute ground truth. Biases shared across the panel members are absorbed into $\mu_\theta$ and propagate into the panel mean. CyclicJudge addresses the within-panel offsets between individual judges, not biases that all panel members share; calibrating the panel itself against human or expert references is complementary work.

\paragraph{Linear model approximation.}
Our variance decomposition assumes a linear random effects model, which treats scores as unbounded and continuous. LLM-as-judge scores are in fact bounded ordinal variables, so the linear model can produce boundary artifacts and heteroscedastic residuals near the scale endpoints. A generalized linear mixed model (GLMM), such as a cumulative-link mixed model, would properly account for the ordinal, bounded nature of the responses. However, GLMMs do not yield the closed-form variance decomposition that our allocation analysis relies on. We therefore retain the linear approximation, noting that it remains reasonable when the scale range is large relative to the observed score dispersion, as in our benchmarks.

\paragraph{Scenario exchangeability.}
The ``maximize scenarios'' recommendation assumes scenarios are interchangeable draws from the dataset; in practice, some scenarios may be more informative than others, and item-selection methods \citep{polo2024tinybenchmarks} could alter the tradeoff.

\paragraph{Balanced design assumption.}
The allocation analysis assumes a balanced design in which every scenario receives the same number of judge calls. Extensions to adaptive or sequential designs remain future work.

\bibliography{references}

@inproceedings{zheng2023judging,
    title = "Judging {LLM}-as-a-Judge with {MT}-{B}ench and {Chatbot} {Arena}",
    author = "Zheng, Lianmin and
      Chiang, Wei-Lin and
      Sheng, Ying and
      Zhuang, Siyuan and
      Wu, Zhanghao and
      Zhuang, Yonghao and
      Lin, Zi and
      Li, Zhuohan and
      Li, Dacheng and
      Xing, Eric P. and
      Zhang, Hao and
      Gonzalez, Joseph E. and
      Stoica, Ion",
    booktitle = "Advances in Neural Information Processing Systems 36 ({NeurIPS} 2023) Datasets and Benchmarks Track",
    volume = "36",
    pages = "46595--46623",
    year = "2023",
    publisher = "Curran Associates, Inc.",
    doi = "10.52202/075280-2020",
    url = "https://proceedings.neurips.cc/paper_files/paper/2023/hash/91f18a1287b398d378ef22505bf41832-Abstract-Datasets_and_Benchmarks.html"
}

@inproceedings{chiang2024chatbot,
    title = "{C}hatbot {A}rena: An Open Platform for Evaluating {LLM}s by Human Preference",
    author = "Chiang, Wei-Lin and
      Zheng, Lianmin and
      Sheng, Ying and
      Angelopoulos, Anastasios Nikolas and
      Li, Tianle and
      Li, Dacheng and
      Zhu, Banghua and
      Zhang, Hao and
      Jordan, Michael and
      Gonzalez, Joseph E. and
      Stoica, Ion",
    editor = "Salakhutdinov, Ruslan and
      Kolter, Zico and
      Heller, Katherine and
      Weller, Adrian and
      Oliver, Nuria and
      Scarlett, Jonathan and
      Berkenkamp, Felix",
    booktitle = "Proceedings of the 41st International Conference on Machine Learning",
    series = "Proceedings of Machine Learning Research",
    volume = "235",
    pages = "8359--8388",
    year = "2024",
    month = "21--27 Jul",
    publisher = "PMLR",
    url = "https://proceedings.mlr.press/v235/chiang24b.html"
}

@article{gu2024survey,
    title = "A survey on {LLM}-as-a-judge",
    author = "Gu, Jiawei and
      Jiang, Xuhui and
      Shi, Zhichao and
      Tan, Hexiang and
      Zhai, Xuehao and
      Xu, Chengjin and
      Li, Wei and
      Shen, Yinghan and
      Ma, Shengjie and
      Liu, Honghao and
      Wang, Saizhuo and
      Zhang, Kun and
      Lin, Zhouchi and
      Zhang, Bowen and
      Ni, Lionel and
      Gao, Wen and
      Wang, Yuanzhuo and
      Guo, Jian",
    journal = "The Innovation",
    volume = "7",
    number = "6",
    pages = "101253",
    year = "2026",
    issn = "2666-6758",
    publisher = "Elsevier",
    doi = "10.1016/j.xinn.2025.101253",
    url = "https://doi.org/10.1016/j.xinn.2025.101253"
}

@inproceedings{wang2023large,
    title = "Large Language Models are not Fair Evaluators",
    author = "Wang, Peiyi  and
      Li, Lei  and
      Chen, Liang  and
      Cai, Zefan  and
      Zhu, Dawei  and
      Lin, Binghuai  and
      Cao, Yunbo  and
      Kong, Lingpeng  and
      Liu, Qi  and
      Liu, Tianyu  and
      Sui, Zhifang",
    editor = "Ku, Lun-Wei  and
      Martins, Andre  and
      Srikumar, Vivek",
    booktitle = "Proceedings of the 62nd Annual Meeting of the Association for Computational Linguistics (Volume 1: Long Papers)",
    month = aug,
    year = "2024",
    address = "Bangkok, Thailand",
    publisher = "Association for Computational Linguistics",
    url = "https://aclanthology.org/2024.acl-long.511/",
    doi = "10.18653/v1/2024.acl-long.511",
    pages = "9440--9450"
}

@inproceedings{dubois2024length,
    title = "Length-Controlled {AlpacaEval}: A Simple Debiasing of Automatic Evaluators",
    author = "Dubois, Yann and
      Liang, Percy and
      Hashimoto, Tatsunori",
    booktitle = "First Conference on Language Modeling ({COLM})",
    month = oct,
    year = "2024",
    address = "Philadelphia, PA, USA",
    url = "https://openreview.net/forum?id=CybBmzWBX0",
    eprint = "2404.04475",
    archivePrefix = "arXiv"
}

@inproceedings{koo2023benchmarking,
    title = "Benchmarking Cognitive Biases in Large Language Models as Evaluators",
    author = "Koo, Ryan  and
      Lee, Minhwa  and
      Raheja, Vipul  and
      Park, Jong Inn  and
      Kim, Zae Myung  and
      Kang, Dongyeop",
    editor = "Ku, Lun-Wei  and
      Martins, Andre  and
      Srikumar, Vivek",
    booktitle = "Findings of the Association for Computational Linguistics: ACL 2024",
    month = aug,
    year = "2024",
    address = "Bangkok, Thailand",
    publisher = "Association for Computational Linguistics",
    url = "https://aclanthology.org/2024.findings-acl.29/",
    doi = "10.18653/v1/2024.findings-acl.29",
    pages = "517--545"
}

@misc{wataoka2024self,
    title = "Self-Preference Bias in {LLM}-as-a-Judge",
    author = "Wataoka, Koki and
      Takahashi, Tsubasa and
      Ri, Ryokan",
    year = "2024",
    eprint = "2410.21819",
    archivePrefix = "arXiv",
    primaryClass = "cs.CL",
    doi = "10.48550/arXiv.2410.21819",
    url = "https://arxiv.org/abs/2410.21819",
    note = "Accepted at the {NeurIPS} 2024 Safe Generative {AI} Workshop"
}

@misc{verga2024replacing,
    title = "Replacing Judges with Juries: Evaluating {LLM} Generations with a Panel of Diverse Models",
    author = "Verga, Pat and
      Hofst{\"a}tter, Sebastian and
      Althammer, Sophia and
      Su, Yixuan and
      Piktus, Aleksandra and
      Arkhangorodsky, Arkady and
      Xu, Minjie and
      White, Naomi and
      Lewis, Patrick",
    year = "2024",
    eprint = "2404.18796",
    archivePrefix = "arXiv",
    primaryClass = "cs.CL",
    doi = "10.48550/arXiv.2404.18796",
    url = "https://arxiv.org/abs/2404.18796"
}

@misc{schroeder2024trust,
    title = "Can You Trust {LLM} Judgments? Reliability of {LLM}-as-a-Judge",
    author = "Schroeder, Kayla and
      Wood-Doughty, Zach",
    year = "2024",
    eprint = "2412.12509",
    archivePrefix = "arXiv",
    primaryClass = "cs.CL",
    doi = "10.48550/arXiv.2412.12509",
    url = "https://arxiv.org/abs/2412.12509"
}

@inproceedings{zhang2025beyond,
    title = "Beyond the Singular: Revealing the Value of Multiple Generations in Benchmark Evaluation",
    author = "Zhang, Wenbo  and
      Cai, Hengrui  and
      Chen, Wenyu",
    editor = "Liakata, Maria  and
      Moreira, Viviane P.  and
      Zhang, Jiajun  and
      Jurgens, David",
    booktitle = "Findings of the {A}ssociation for {C}omputational {L}inguistics: {ACL} 2026",
    month = jul,
    year = "2026",
    address = "San Diego, California, United States",
    publisher = "Association for Computational Linguistics",
    url = "https://aclanthology.org/2026.findings-acl.488/",
    doi = "10.18653/v1/2026.findings-acl.488",
    pages = "10033--10043",
    ISBN = "979-8-89176-395-1"
}

@misc{saha2026budget,
    title = "{LLM}-as-Judge on a Budget",
    author = "Saha, Aadirupa and
      Wagde, Aniket and
      Kveton, Branislav",
    year = "2026",
    eprint = "2602.15481",
    archivePrefix = "arXiv",
    primaryClass = "cs.LG",
    doi = "10.48550/arXiv.2602.15481",
    url = "https://arxiv.org/abs/2602.15481"
}

@inproceedings{dorner2024limits,
    title = "Limits to scalable evaluation at the frontier: {LLM} as judge won't beat twice the data",
    author = "Dorner, Florian E. and
      Nastl, Vivian Y. and
      Hardt, Moritz",
    booktitle = "The Thirteenth International Conference on Learning Representations ({ICLR})",
    pages = "26467--26491",
    month = apr,
    year = "2025",
    address = "Singapore",
    publisher = "OpenReview.net",
    eprint = "2410.13341",
    archivePrefix = "arXiv",
    primaryClass = "cs.LG",
    url = "https://openreview.net/forum?id=NO6Tv6QcDs"
}

@inproceedings{polo2024tinybenchmarks,
    title = "tiny{B}enchmarks: evaluating {LLM}s with fewer examples",
    author = "Maia Polo, Felipe and
      Weber, Lucas and
      Choshen, Leshem and
      Sun, Yuekai and
      Xu, Gongjun and
      Yurochkin, Mikhail",
    booktitle = "Proceedings of the 41st International Conference on Machine Learning",
    editor = "Salakhutdinov, Ruslan and
      Kolter, Zico and
      Heller, Katherine and
      Weller, Adrian and
      Oliver, Nuria and
      Scarlett, Jonathan and
      Berkenkamp, Felix",
    volume = "235",
    series = "Proceedings of Machine Learning Research",
    month = "21--27 Jul",
    year = "2024",
    pages = "34303--34326",
    publisher = "PMLR",
    url = "https://proceedings.mlr.press/v235/maia-polo24a.html"
}

@article{mizrahi2024state,
    title = "State of What Art? A Call for Multi-Prompt {LLM} Evaluation",
    author = "Mizrahi, Moran  and
      Kaplan, Guy  and
      Malkin, Dan  and
      Dror, Rotem  and
      Shahaf, Dafna  and
      Stanovsky, Gabriel",
    journal = "Transactions of the Association for Computational Linguistics",
    volume = "12",
    year = "2024",
    address = "Cambridge, MA",
    publisher = "MIT Press",
    url = "https://aclanthology.org/2024.tacl-1.52/",
    doi = "10.1162/tacl_a_00681",
    pages = "933--949"
}

@misc{yang2024qwen25,
    title = "{Qwen2.5} Technical Report",
    author = "Yang, An and
      Yang, Baosong and
      Zhang, Beichen and
      Hui, Binyuan and
      Zheng, Bo and
      Yu, Bowen and
      Li, Chengyuan and
      Liu, Dayiheng and
      Huang, Fei and
      Wei, Haoran and
      Lin, Huan and
      Yang, Jian and
      Tu, Jianhong and
      Zhang, Jianwei and
      Yang, Jianxin and
      Yang, Jiaxi and
      Zhou, Jingren and
      Lin, Junyang and
      Dang, Kai and
      Lu, Keming and
      Bao, Keqin and
      Yang, Kexin and
      Yu, Le and
      Li, Mei and
      Xue, Mingfeng and
      Zhang, Pei and
      Zhu, Qin and
      Men, Rui and
      Lin, Runji and
      Li, Tianhao and
      Tang, Tianyi and
      Xia, Tingyu and
      Ren, Xingzhang and
      Ren, Xuancheng and
      Fan, Yang and
      Su, Yang and
      Zhang, Yichang and
      Wan, Yu and
      Liu, Yuqiong and
      Cui, Zeyu and
      Zhang, Zhenru and
      Qiu, Zihan",
    year = "2024",
    eprint = "2412.15115",
    archivePrefix = "arXiv",
    primaryClass = "cs.CL",
    doi = "10.48550/arXiv.2412.15115",
    url = "https://arxiv.org/abs/2412.15115"
}

@misc{grattafiori2024llama,
    title = {The {L}lama 3 Herd of Models},
    author = {Grattafiori, Aaron and
      Dubey, Abhimanyu and
      Jauhri, Abhinav and
      Pandey, Abhinav and
      Kadian, Abhishek and
      Al-Dahle, Ahmad and
      Letman, Aiesha and
      Mathur, Akhil and
      Schelten, Alan and
      Vaughan, Alex and
      Yang, Amy and
      Fan, Angela and
      Goyal, Anirudh and
      Hartshorn, Anthony and
      Yang, Aobo and
      Mitra, Archi and
      Sravankumar, Archie and
      Korenev, Artem and
      Hinsvark, Arthur and
      Rao, Arun and
      Zhang, Aston and
      Rodriguez, Aurelien and
      Gregerson, Austen and
      Spataru, Ava and
      Roziere, Baptiste and
      Biron, Bethany and
      Tang, Binh and
      Chern, Bobbie and
      Caucheteux, Charlotte and
      Nayak, Chaya and
      Bi, Chloe and
      Marra, Chris and
      McConnell, Chris and
      Keller, Christian and
      Touret, Christophe and
      Wu, Chunyang and
      Wong, Corinne and
      Ferrer, Cristian Canton and
      Nikolaidis, Cyrus and
      Allonsius, Damien and
      Song, Daniel and
      Pintz, Danielle and
      Livshits, Danny and
      Wyatt, Danny and
      Esiobu, David and
      Choudhary, Dhruv and
      Mahajan, Dhruv and
      Garcia-Olano, Diego and
      Perino, Diego and
      Hupkes, Dieuwke and
      Lakomkin, Egor and
      AlBadawy, Ehab and
      Lobanova, Elina and
      Dinan, Emily and
      Smith, Eric Michael and
      Radenovic, Filip and
      Guzm{\'a}n, Francisco and
      Zhang, Frank and
      Synnaeve, Gabriel and
      Lee, Gabrielle and
      Anderson, Georgia Lewis and
      Thattai, Govind and
      Nail, Graeme and
      Mialon, Gregoire and
      Pang, Guan and
      Cucurell, Guillem and
      Nguyen, Hailey and
      Korevaar, Hannah and
      Xu, Hu and
      Touvron, Hugo and
      Zarov, Iliyan and
      Ibarra, Imanol Arrieta and
      Kloumann, Isabel and
      Misra, Ishan and
      Evtimov, Ivan and
      Zhang, Jack and
      Copet, Jade and
      Lee, Jaewon and
      Geffert, Jan and
      Vranes, Jana and
      Park, Jason and
      Mahadeokar, Jay and
      Shah, Jeet and
      van der Linde, Jelmer and
      Billock, Jennifer and
      Hong, Jenny and
      Lee, Jenya and
      Fu, Jeremy and
      Chi, Jianfeng and
      Huang, Jianyu and
      Liu, Jiawen and
      Wang, Jie and
      Yu, Jiecao and
      Bitton, Joanna and
      Spisak, Joe and
      Park, Jongsoo and
      Rocca, Joseph and
      Johnstun, Joshua and
      Saxe, Joshua and
      Jia, Junteng and
      Alwala, Kalyan Vasuden and
      Prasad, Karthik and
      Upasani, Kartikeya and
      Plawiak, Kate and
      Li, Ke and
      Heafield, Kenneth and
      Stone, Kevin and
      El-Arini, Khalid and
      Iyer, Krithika and
      Malik, Kshitiz and
      Chiu, Kuenley and
      Bhalla, Kunal and
      Lakhotia, Kushal and
      Rantala-Yeary, Lauren and
      van der Maaten, Laurens and
      Chen, Lawrence and
      Tan, Liang and
      Jenkins, Liz and
      Martin, Louis and
      Madaan, Lovish and
      Malo, Lubo and
      Blecher, Lukas and
      Landzaat, Lukas and
      de Oliveira, Luke and
      Muzzi, Madeline and
      Pasupuleti, Mahesh and
      Singh, Mannat and
      Paluri, Manohar and
      Kardas, Marcin and
      Tsimpoukelli, Maria and
      Oldham, Mathew and
      Rita, Mathieu and
      Pavlova, Maya and
      Kambadur, Melanie and
      Lewis, Mike and
      Si, Min and
      Singh, Mitesh Kumar and
      Hassan, Mona and
      Goyal, Naman and
      Torabi, Narjes and
      Bashlykov, Nikolay and
      Bogoychev, Nikolay and
      Chatterji, Niladri and
      Zhang, Ning and
      Duchenne, Olivier and
      {\c{C}}elebi, Onur and
      Alrassy, Patrick and
      Zhang, Pengchuan and
      Li, Pengwei and
      Vasic, Petar and
      Weng, Peter and
      Bhargava, Prajjwal and
      Dubal, Pratik and
      Krishnan, Praveen and
      Koura, Punit Singh and
      Xu, Puxin and
      He, Qing and
      Dong, Qingxiao and
      Srinivasan, Ragavan and
      Ganapathy, Raj and
      Calderer, Ramon and
      Cabral, Ricardo Silveira and
      Stojnic, Robert and
      Raileanu, Roberta and
      Maheswari, Rohan and
      Girdhar, Rohit and
      Patel, Rohit and
      Sauvestre, Romain and
      Polidoro, Ronnie and
      Sumbaly, Roshan and
      Taylor, Ross and
      Silva, Ruan and
      Hou, Rui and
      Wang, Rui and
      Hosseini, Saghar and
      Chennabasappa, Sahana and
      Singh, Sanjay and
      Bell, Sean and
      Kim, Seohyun Sonia and
      Edunov, Sergey and
      Nie, Shaoliang and
      Narang, Sharan and
      Raparthy, Sharath and
      Shen, Sheng and
      Wan, Shengye and
      Bhosale, Shruti and
      Zhang, Shun and
      Vandenhende, Simon and
      Batra, Soumya and
      Whitman, Spencer and
      Sootla, Sten and
      Collot, Stephane and
      Gururangan, Suchin and
      Borodinsky, Sydney and
      Herman, Tamar and
      Fowler, Tara and
      Sheasha, Tarek and
      Georgiou, Thomas and
      Scialom, Thomas and
      Speckbacher, Tobias and
      Mihaylov, Todor and
      Xiao, Tong and
      Karn, Ujjwal and
      Goswami, Vedanuj and
      Gupta, Vibhor and
      Ramanathan, Vignesh and
      Kerkez, Viktor and
      Gonguet, Vincent and
      Do, Virginie and
      Vogeti, Vish and
      Albiero, V{\'i}tor and
      Petrovic, Vladan and
      Chu, Weiwei and
      Xiong, Wenhan and
      Fu, Wenyin and
      Meers, Whitney and
      Martinet, Xavier and
      Wang, Xiaodong and
      Wang, Xiaofang and
      Tan, Xiaoqing Ellen and
      Xia, Xide and
      Xie, Xinfeng and
      Jia, Xuchao and
      Wang, Xuewei and
      Goldschlag, Yaelle and
      Gaur, Yashesh and
      Babaei, Yasmine and
      Wen, Yi and
      Song, Yiwen and
      Zhang, Yuchen and
      Li, Yue and
      Mao, Yuning and
      Coudert, Zacharie Delpierre and
      Yan, Zheng and
      Chen, Zhengxing and
      Papakipos, Zoe and
      Singh, Aaditya and
      Srivastava, Aayushi and
      Jain, Abha and
      Kelsey, Adam and
      Shajnfeld, Adam and
      Gangidi, Adithya and
      Victoria, Adolfo and
      Goldstand, Ahuva and
      Menon, Ajay and
      Sharma, Ajay and
      Boesenberg, Alex and
      Baevski, Alexei and
      Feinstein, Allie and
      Kallet, Amanda and
      Sangani, Amit and
      Teo, Amos and
      Yunus, Anam and
      Lupu, Andrei and
      Alvarado, Andres and
      Caples, Andrew and
      Gu, Andrew and
      Ho, Andrew and
      Poulton, Andrew and
      Ryan, Andrew and
      Ramchandani, Ankit and
      Dong, Annie and
      Franco, Annie and
      Goyal, Anuj and
      Saraf, Aparajita and
      Chowdhury, Arkabandhu and
      Gabriel, Ashley and
      Bharambe, Ashwin and
      Eisenman, Assaf and
      Yazdan, Azadeh and
      James, Beau and
      Maurer, Ben and
      Leonhardi, Benjamin and
      Huang, Bernie and
      Loyd, Beth and
      De Paola, Beto and
      Paranjape, Bhargavi and
      Liu, Bing and
      Wu, Bo and
      Ni, Boyu and
      Hancock, Braden and
      Wasti, Bram and
      Spence, Brandon and
      Stojkovic, Brani and
      Gamido, Brian and
      Montalvo, Britt and
      Parker, Carl and
      Burton, Carly and
      Mejia, Catalina and
      Liu, Ce and
      Wang, Changhan and
      Kim, Changkyu and
      Zhou, Chao and
      Hu, Chester and
      Chu, Ching-Hsiang and
      Cai, Chris and
      Tindal, Chris and
      Feichtenhofer, Christoph and
      Gao, Cynthia and
      Civin, Damon and
      Beaty, Dana and
      Kreymer, Daniel and
      Li, Daniel and
      Adkins, David and
      Xu, David and
      Testuggine, Davide and
      David, Delia and
      Parikh, Devi and
      Liskovich, Diana and
      Foss, Didem and
      Wang, Dingkang and
      Le, Duc and
      Holland, Dustin and
      Dowling, Edward and
      Jamil, Eissa and
      Montgomery, Elaine and
      Presani, Eleonora and
      Hahn, Emily and
      Wood, Emily and
      Le, Eric-Tuan and
      Brinkman, Erik and
      Arcaute, Esteban and
      Dunbar, Evan and
      Smothers, Evan and
      Sun, Fei and
      Kreuk, Felix and
      Tian, Feng and
      Kokkinos, Filippos and
      Ozgenel, Firat and
      Caggioni, Francesco and
      Kanayet, Frank and
      Seide, Frank and
      Florez, Gabriela Medina and
      Schwarz, Gabriella and
      Badeer, Gada and
      Swee, Georgia and
      Halpern, Gil and
      Herman, Grant and
      Sizov, Grigory and
      Zhang, Guangyi and
      Lakshminarayanan, Guna and
      Inan, Hakan and
      Shojanazeri, Hamid and
      Zou, Han and
      Wang, Hannah and
      Zha, Hanwen and
      Habeeb, Haroun and
      Rudolph, Harrison and
      Suk, Helen and
      Aspegren, Henry and
      Goldman, Hunter and
      Zhan, Hongyuan and
      Damlaj, Ibrahim and
      Molybog, Igor and
      Tufanov, Igor and
      Leontiadis, Ilias and
      Veliche, Irina-Elena and
      Gat, Itai and
      Weissman, Jake and
      Geboski, James and
      Kohli, James and
      Lam, Janice and
      Asher, Japhet and
      Gaya, Jean-Baptiste and
      Marcus, Jeff and
      Tang, Jeff and
      Chan, Jennifer and
      Zhen, Jenny and
      Reizenstein, Jeremy and
      Teboul, Jeremy and
      Zhong, Jessica and
      Jin, Jian and
      Yang, Jingyi and
      Cummings, Joe and
      Carvill, Jon and
      Shepard, Jon and
      McPhie, Jonathan and
      Torres, Jonathan and
      Ginsburg, Josh and
      Wang, Junjie and
      Wu, Kai and
      U, Kam Hou and
      Saxena, Karan and
      Khandelwal, Kartikay and
      Zand, Katayoun and
      Matosich, Kathy and
      Veeraraghavan, Kaushik and
      Michelena, Kelly and
      Li, Keqian and
      Jagadeesh, Kiran and
      Huang, Kun and
      Chawla, Kunal and
      Huang, Kyle and
      Chen, Lailin and
      Garg, Lakshya and
      A, Lavender and
      Silva, Leandro and
      Bell, Lee and
      Zhang, Lei and
      Guo, Liangpeng and
      Yu, Licheng and
      Moshkovich, Liron and
      Wehrstedt, Luca and
      Khabsa, Madian and
      Avalani, Manav and
      Bhatt, Manish and
      Mankus, Martynas and
      Hasson, Matan and
      Lennie, Matthew and
      Reso, Matthias and
      Groshev, Maxim and
      Naumov, Maxim and
      Lathi, Maya and
      Keneally, Meghan and
      Liu, Miao and
      Seltzer, Michael L. and
      Valko, Michal and
      Restrepo, Michelle and
      Patel, Mihir and
      Vyatskov, Mik and
      Samvelyan, Mikayel and
      Clark, Mike and
      Macey, Mike and
      Wang, Mike and
      Hermoso, Miquel Jubert and
      Metanat, Mo and
      Rastegari, Mohammad and
      Bansal, Munish and
      Santhanam, Nandhini and
      Parks, Natascha and
      White, Natasha and
      Bawa, Navyata and
      Singhal, Nayan and
      Egebo, Nick and
      Usunier, Nicolas and
      Mehta, Nikhil and
      Laptev, Nikolay Pavlovich and
      Dong, Ning and
      Cheng, Norman and
      Chernoguz, Oleg and
      Hart, Olivia and
      Salpekar, Omkar and
      Kalinli, Ozlem and
      Kent, Parkin and
      Parekh, Parth and
      Saab, Paul and
      Balaji, Pavan and
      Rittner, Pedro and
      Bontrager, Philip and
      Roux, Pierre and
      Dollar, Piotr and
      Zvyagina, Polina and
      Ratanchandani, Prashant and
      Yuvraj, Pritish and
      Liang, Qian and
      Alao, Rachad and
      Rodriguez, Rachel and
      Ayub, Rafi and
      Murthy, Raghotham and
      Nayani, Raghu and
      Mitra, Rahul and
      Parthasarathy, Rangaprabhu and
      Li, Raymond and
      Hogan, Rebekkah and
      Battey, Robin and
      Wang, Rocky and
      Howes, Russ and
      Rinott, Ruty and
      Mehta, Sachin and
      Siby, Sachin and
      Bondu, Sai Jayesh and
      Datta, Samyak and
      Chugh, Sara and
      Hunt, Sara and
      Dhillon, Sargun and
      Sidorov, Sasha and
      Pan, Satadru and
      Mahajan, Saurabh and
      Verma, Saurabh and
      Yamamoto, Seiji and
      Ramaswamy, Sharadh and
      Lindsay, Shaun and
      Lindsay, Shaun and
      Feng, Sheng and
      Lin, Shenghao and
      Zha, Shengxin Cindy and
      Patil, Shishir and
      Shankar, Shiva and
      Zhang, Shuqiang and
      Zhang, Shuqiang and
      Wang, Sinong and
      Agarwal, Sneha and
      Sajuyigbe, Soji and
      Chintala, Soumith and
      Max, Stephanie and
      Chen, Stephen and
      Kehoe, Steve and
      Satterfield, Steve and
      Govindaprasad, Sudarshan and
      Gupta, Sumit and
      Deng, Summer and
      Cho, Sungmin and
      Virk, Sunny and
      Subramanian, Suraj and
      Choudhury, Sy and
      Goldman, Sydney and
      Remez, Tal and
      Glaser, Tamar and
      Best, Tamara and
      Koehler, Thilo and
      Robinson, Thomas and
      Li, Tianhe and
      Zhang, Tianjun and
      Matthews, Tim and
      Chou, Timothy and
      Shaked, Tzook and
      Vontimitta, Varun and
      Ajayi, Victoria and
      Montanez, Victoria and
      Mohan, Vijai and
      Kumar, Vinay Satish and
      Mangla, Vishal and
      Ionescu, Vlad and
      Poenaru, Vlad and
      Mihailescu, Vlad Tiberiu and
      Ivanov, Vladimir and
      Li, Wei and
      Wang, Wenchen and
      Jiang, Wenwen and
      Bouaziz, Wes and
      Constable, Will and
      Tang, Xiaocheng and
      Wu, Xiaojian and
      Wang, Xiaolan and
      Wu, Xilun and
      Gao, Xinbo and
      Kleinman, Yaniv and
      Chen, Yanjun and
      Hu, Ye and
      Jia, Ye and
      Qi, Ye and
      Li, Yenda and
      Zhang, Yilin and
      Zhang, Ying and
      Adi, Yossi and
      Nam, Youngjin and
      Wang, Yu and
      Zhao, Yu and
      Hao, Yuchen and
      Qian, Yundi and
      Li, Yunlu and
      He, Yuzi and
      Rait, Zach and
      DeVito, Zachary and
      Rosnbrick, Zef and
      Wen, Zhaoduo and
      Yang, Zhenyu and
      Zhao, Zhiwei and
      Ma, Zhiyu},
    year = {2024},
    eprint = {2407.21783},
    archivePrefix = {arXiv},
    primaryClass = {cs.AI},
    doi = {10.48550/arXiv.2407.21783},
    url = {https://arxiv.org/abs/2407.21783}
}

@book{searle1992variance,
    title = {Variance Components},
    author = {Searle, Shayle R. and Casella, George and McCulloch, Charles E.},
    year = {1992},
    series = {Wiley Series in Probability and Mathematical Statistics},
    publisher = {Wiley},
    address = {New York},
    isbn = {9780471621621},
    doi = {10.1002/9780470316856},
    url = {https://doi.org/10.1002/9780470316856}
}

@book{brennan2001generalizability,
    title = {Generalizability Theory},
    author = {Brennan, Robert L.},
    year = {2001},
    publisher = {Springer},
    address = {New York, NY},
    isbn = {978-0-387-95282-6},
    doi = {10.1007/978-1-4757-3456-0},
    url = {https://doi.org/10.1007/978-1-4757-3456-0}
}

@book{shavelson1991generalizability,
    title = {Generalizability Theory: A Primer},
    author = {Shavelson, Richard J. and Webb, Noreen M.},
    year = {1991},
    publisher = {Sage Publications},
    address = {Newbury Park, CA},
    series = {Measurement Methods for the Social Sciences},
    volume = {1},
    isbn = {978-0-8039-3745-1},
    url = {https://books.google.com/books?id=gzyKRZWrm9EC}
}

@inproceedings{shi2024judging,
    title = "Judging the Judges: A Systematic Study of Position Bias in {LLM}-as-a-Judge",
    author = "Shi, Lin  and
      Ma, Chiyu  and
      Liang, Wenhua  and
      Diao, Xingjian  and
      Ma, Weicheng  and
      Vosoughi, Soroush",
    editor = "Inui, Kentaro  and
      Sakti, Sakriani  and
      Wang, Haofen  and
      Wong, Derek F.  and
      Bhattacharyya, Pushpak  and
      Banerjee, Biplab  and
      Ekbal, Asif  and
      Chakraborty, Tanmoy  and
      Singh, Dhirendra Pratap",
    booktitle = "Proceedings of the 14th International Joint Conference on Natural Language Processing and the 4th Conference of the Asia-Pacific Chapter of the Association for Computational Linguistics",
    month = dec,
    year = "2025",
    address = "Mumbai, India",
    publisher = "The Asian Federation of Natural Language Processing and The Association for Computational Linguistics",
    url = "https://aclanthology.org/2025.ijcnlp-long.18/",
    doi = "10.18653/v1/2025.ijcnlp-long.18",
    pages = "292--314",
    ISBN = "979-8-89176-298-5"
}

@misc{madaan2024quantifying,
    title = "Quantifying Variance in Evaluation Benchmarks",
    author = "Madaan, Lovish and
      Singh, Aaditya K. and
      Schaeffer, Rylan and
      Poulton, Andrew and
      Koyejo, Sanmi and
      Stenetorp, Pontus and
      Narang, Sharan and
      Hupkes, Dieuwke",
    year = "2024",
    eprint = "2406.10229",
    archivePrefix = "arXiv",
    primaryClass = "cs.LG",
    url = "https://arxiv.org/abs/2406.10229",
    doi = "10.48550/arXiv.2406.10229"
}

@misc{miller2024error,
    title = "Adding Error Bars to Evals: A Statistical Approach to Language Model Evaluations",
    author = "Miller, Evan",
    year = "2024",
    eprint = "2411.00640",
    archivePrefix = "arXiv",
    primaryClass = "stat.AP",
    doi = "10.48550/arXiv.2411.00640",
    url = "https://arxiv.org/abs/2411.00640"
}

@misc{alvarado2025repetitions,
    title = "Do Repetitions Matter? Strengthening Reliability in {LLM} Evaluations",
    author = "Alvarado Gonzalez, Miguel Angel and
      Bruno Hernandez, Michelle and
      Pe{\~n}aloza Perez, Miguel Angel and
      Lopez Orozco, Bruno and
      Cruz Soto, Jesus Tadeo and
      Malagon, Sandra",
    year = "2025",
    eprint = "2509.24086",
    archivePrefix = "arXiv",
    primaryClass = "cs.AI",
    doi = "10.48550/arXiv.2509.24086",
    url = "https://arxiv.org/abs/2509.24086"
}

@misc{luettgau2025hibayes,
    title = "{HiBayES}: A Hierarchical {B}ayesian Modeling Framework for {AI} Evaluation Statistics",
    author = "Luettgau, Lennart and
      Coppock, Harry and
      Dubois, Magda and
      Summerfield, Christopher and
      Ududec, Cozmin",
    year = "2025",
    eprint = "2505.05602",
    archivePrefix = "arXiv",
    primaryClass = "cs.AI",
    doi = "10.48550/arXiv.2505.05602",
    url = "https://arxiv.org/abs/2505.05602"
}

@misc{pombal2025mindeval,
    title = {{MindEval: Benchmarking Language Models on Multi-turn Mental Health Support}},
    author = {Pombal, Jos{\'e} and
      D'Eon, Maya and
      Guerreiro, Nuno M. and
      Martins, Pedro Henrique and
      Farinhas, Ant{\'o}nio and
      Rei, Ricardo},
    year = {2025},
    month = nov,
    eprint = {2511.18491},
    archivePrefix = {arXiv},
    primaryClass = {cs.CL},
    doi = {10.48550/arXiv.2511.18491},
    url = {https://arxiv.org/abs/2511.18491}
}

@inproceedings{matsutani2025rl,
    title = "{RL} Squeezes, {SFT} Expands: A Comparative Study of Reasoning {LLM}s",
    author = "Matsutani, Kohsei and
      Takashiro, Shota and
      Minegishi, Gouki and
      Kojima, Takeshi and
      Iwasawa, Yusuke and
      Matsuo, Yutaka",
    booktitle = "The Fourteenth International Conference on Learning Representations ({ICLR})",
    year = "2026",
    month = apr,
    volume = "2026",
    pages = "141482--141526",
    url = "https://proceedings.iclr.cc/paper_files/paper/2026/hash/e52554a70e0df57a0bea11d1eca0c9b5-Abstract-Conference.html",
    eprint = "2509.21128",
    archivePrefix = "arXiv"
}

@book{fisher1925statistical,
    title = {Statistical Methods for Research Workers},
    author = {Fisher, R. A.},
    year = {1925},
    edition = {1st},
    series = {Biological Monographs and Manuals},
    publisher = {Oliver and Boyd},
    address = {Edinburgh}
}

@inproceedings{perlitz2023efficient,
    title = "Efficient Benchmarking (of Language Models)",
    author = "Perlitz, Yotam  and
      Bandel, Elron  and
      Gera, Ariel  and
      Arviv, Ofir  and
      Ein-Dor, Liat  and
      Shnarch, Eyal  and
      Slonim, Noam  and
      Shmueli-Scheuer, Michal  and
      Choshen, Leshem",
    editor = "Duh, Kevin  and
      Gomez, Helena  and
      Bethard, Steven",
    booktitle = "Proceedings of the 2024 Conference of the North American Chapter of the Association for Computational Linguistics: Human Language Technologies (Volume 1: Long Papers)",
    month = jun,
    year = "2024",
    address = "Mexico City, Mexico",
    publisher = "Association for Computational Linguistics",
    url = "https://aclanthology.org/2024.naacl-long.139/",
    doi = "10.18653/v1/2024.naacl-long.139",
    pages = "2519--2536"
}

\appendix

\section{Judge Bias Estimates}
\label{app:judge_biases}

Figure~\ref{fig:judge_biases} shows the estimated per-judge bias for each judge--model pair on both benchmarks. A pronounced self-preference is visible along the diagonal: on MindEval, Claude Sonnet 4.6's self-bias of $+0.412$ is the \emph{only} positive entry in its otherwise uniformly negative row, and Qwen 2.5 7B's self-bias is the largest in its row on both benchmarks.

Judge temperament is not stable across benchmarks: Gemini 3 Flash is the most lenient judge on MT-Bench (all biases positive), yet becomes one of the harshest on MindEval, suggesting that judge personality is benchmark-dependent. Strictness does not track model capability. On MT-Bench the uniformly negative judges are GPT-5.2 and Claude Sonnet 4.6, and the uniformly positive ones are Llama 3.3 70B and Gemini 3 Flash, so the strongest model by consensus (Gemini, Table~\ref{tab:components}) is also the single most lenient judge, while Qwen 2.5 7B, the weakest model, is neither: its row mixes a large positive self-bias with negative biases toward the three strongest models. Finally, the harshest judge--model interactions concentrate on capability-mismatched pairs: GPT-5.2 and Claude penalize Qwen by over one full scale point on MT-Bench ($-1.04$ and $-1.14$), whereas their biases toward each other are markedly smaller ($-0.71$ and $-0.23$).

\begin{figure}[ht]
\centering
\begin{subfigure}[b]{\columnwidth}
\centering
\includegraphics[width=\columnwidth]{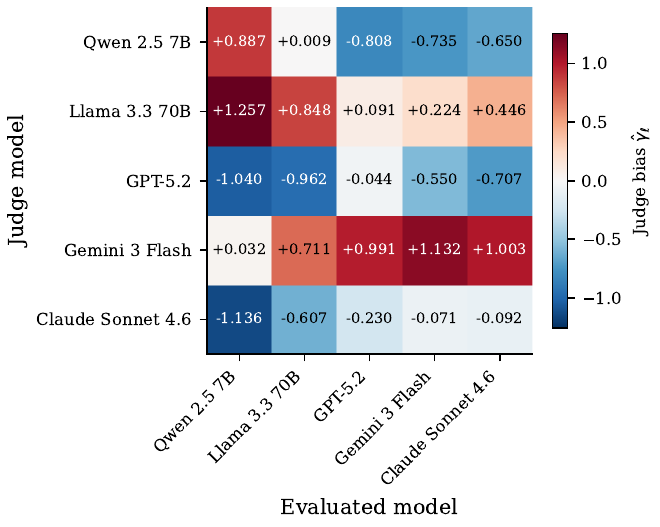}
\caption{MT-Bench}
\end{subfigure}
\vspace{0.5em}
\begin{subfigure}[b]{\columnwidth}
\centering
\includegraphics[width=\columnwidth]{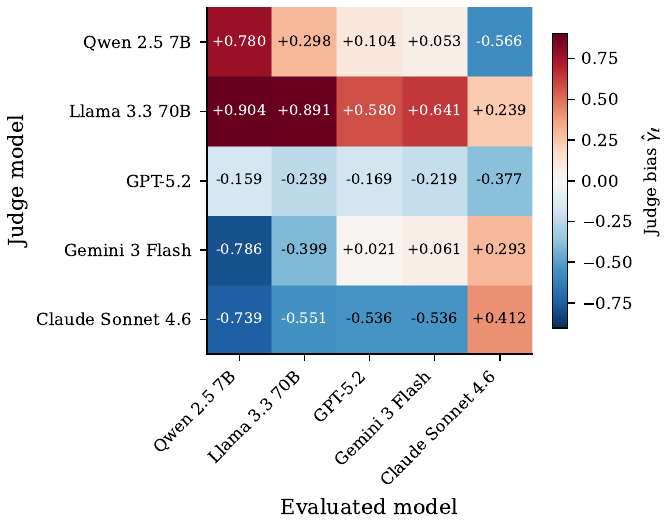}
\caption{MindEval}
\end{subfigure}
\caption{Estimated judge biases $\hat{\gamma}_\ell = \bar{X}_{\cdot\cdot\ell} - \bar{X}_{\cdots}$ for each judge--model pair.}
\label{fig:judge_biases}
\end{figure}

\section{Variance Component Estimates}
\label{app:variance_tables}

\begin{table}[ht]
\centering
\caption{Variance component estimates. \emph{Per-observation variances} and their \emph{contribution to $\Var(\bar{X})$} at the default operating point ($m{=}1$, $K{=}1$).}
\label{tab:components}
\smallskip
{\small (a) MT-Bench ($n{=}80$)}
\vspace{2pt}

\resizebox{\columnwidth}{!}{%
\begin{tabular}{@{}lccccc@{}}
\toprule
\textbf{Component} & \textbf{Qwen} & \textbf{Llama} & \textbf{GPT} & \textbf{Gemini} & \textbf{Claude} \\
\midrule
$\hat{\mu}_\theta$ (mean score) & 7.13 & 7.87 & 8.72 & 8.75 & 8.61 \\
\midrule
\multicolumn{6}{@{}l}{\emph{Per-observation variance}} \\
$\hat{\sigma}^2_\alpha$ (scenario)     & 1.530 & 0.882 & 0.634 & 0.408 & 0.393 \\
$\hat{\sigma}^2_\beta$ (generation)  & 0.266 & 0.238 & 0.076 & 0.000 & 0.000 \\
$\hat{\sigma}^2_\gamma$ (judge)    & 0.947 & 0.503 & 0.339 & 0.435 & 0.427 \\
$\hat{\sigma}^2_\varepsilon$ (residual) & 1.486 & 1.130 & 0.564 & 0.661 & 0.830 \\
\midrule
\multicolumn{6}{@{}l}{\emph{Contribution to $\Var(\bar{X})$ ($\times 10^{-3}$)}} \\
$\hat{\sigma}^2_\alpha / n$            & 19.13 & 11.03 & 7.92 & 5.11 & 4.92 \\
$\hat{\sigma}^2_\beta / nm$            & 3.33 & 2.97 & 0.94 & 0.00 & 0.00 \\
$\hat{\sigma}^2_\gamma \cdot \text{FPC} / K$  & 946.50 & 502.50 & 339.00 & 435.40 & 426.60 \\
$\hat{\sigma}^2_\varepsilon / nmK$     & 18.57 & 14.12 & 7.05 & 8.27 & 10.38 \\
\bottomrule
\end{tabular}}

\medskip
{\small (b) MindEval ($n{=}50$)}
\vspace{2pt}

\resizebox{\columnwidth}{!}{%
\begin{tabular}{@{}lccccc@{}}
\toprule
\textbf{Component} & \textbf{Qwen} & \textbf{Llama} & \textbf{GPT} & \textbf{Gemini} & \textbf{Claude} \\
\midrule
$\hat{\mu}_\theta$ (mean score) & 2.478 & 3.400 & 3.931 & 3.846 & 4.431 \\
\midrule
\multicolumn{6}{@{}l}{\emph{Per-observation variance}} \\
$\hat{\sigma}^2_\alpha$ (scenario)     & 0.034 & 0.014 & 0.021 & 0.015 & 0.005 \\
$\hat{\sigma}^2_\beta$ (generation)  & 0.100 & 0.019 & 0.047 & 0.016 & 0.002 \\
$\hat{\sigma}^2_\gamma$ (judge)    & 0.523 & 0.280 & 0.132 & 0.150 & 0.155 \\
$\hat{\sigma}^2_\varepsilon$ (residual) & 0.159 & 0.069 & 0.207 & 0.103 & 0.061 \\
\midrule
\multicolumn{6}{@{}l}{\emph{Contribution to $\Var(\bar{X})$ ($\times 10^{-3}$)}} \\
$\hat{\sigma}^2_\alpha / n$            & 0.67 & 0.28 & 0.43 & 0.29 & 0.09 \\
$\hat{\sigma}^2_\beta / nm$            & 2.01 & 0.38 & 0.93 & 0.31 & 0.04 \\
$\hat{\sigma}^2_\gamma \cdot \text{FPC} / K$  & 522.50 & 280.20 & 132.00 & 150.20 & 154.80 \\
$\hat{\sigma}^2_\varepsilon / nmK$     & 3.18 & 1.37 & 4.14 & 2.05 & 1.23 \\
\bottomrule
\end{tabular}}
\end{table}

Table~\ref{tab:components} reports the full variance component estimates. The generation variance entries $\hat{\sigma}^2_\beta = 0.000$ (e.g., Gemini and Claude on MT-Bench) arise because the method-of-moments estimator $\hat{\sigma}_\beta^2 = (\MS_G - \MS_W)/K$ (Appendix~\ref{app:estimation}) can return a negative value when the between-generation mean square is smaller than the residual mean square. Because a variance is non-negative by definition, negative estimates are set to zero based on the standard truncation convention \citep{searle1992variance}. These zeros, therefore, indicate that generation-to-generation variability is negligible relative to residual noise for these models, not that it is exactly absent.

\section{Proof of Proposition~\ref{prop:variance}}
\label{app:derivation}

We prove the variance decomposition in three steps:

\paragraph{Step 1: Grand mean decomposition.}

Starting from $X_{ij\ell} = \mu_\theta + \alpha_i + \beta_{ij} + \gamma_\ell + \varepsilon_{ij\ell}$, we average hierarchically.

\emph{Average over $K$ judges within generation $(i,j)$.} The terms $\mu_\theta$, $\alpha_i$, and $\beta_{ij}$ are constant across judges:
\begin{align}
\bar{X}_{ij\cdot} &= \frac{1}{K}\sum_{\ell \in S_J} X_{ij\ell} \nonumber \\
&= \mu_\theta + \alpha_i + \beta_{ij} + \underbrace{\frac{1}{K}\sum_{\ell \in S_J} \gamma_\ell}_{\bar{\gamma}} + \underbrace{\frac{1}{K}\sum_{\ell \in S_J} \varepsilon_{ij\ell}}_{\bar{\varepsilon}_{ij}}
\end{align}
The judge bias mean $\bar{\gamma}$ is the same for every $(i,j)$ cell because the same $S_J$ evaluates all cells (crossed design).

\emph{Average over $m$ generations within scenario $i$.}
\begin{align}
\bar{X}_{i\cdot\cdot} &= \mu_\theta + \alpha_i + \underbrace{\frac{1}{m}\sum_{j=1}^m \beta_{ij}}_{\bar{\beta}_i} + \bar{\gamma} + \underbrace{\frac{1}{mK}\sum_{j,\ell} \varepsilon_{ij\ell}}_{\bar{\varepsilon}_{i\cdot}}
\end{align}

\emph{Average over $n$ scenarios.}
\begin{align}
\bar{X} &= \mu_\theta + \bar{\gamma} + \underbrace{\frac{1}{n}\sum_{i=1}^n \alpha_i}_{\bar{\alpha}} + \underbrace{\frac{1}{nm}\sum_{i,j} \beta_{ij}}_{\bar{\beta}} \nonumber \\
&\quad + \underbrace{\frac{1}{nmK}\sum_{i,j,\ell} \varepsilon_{ij\ell}}_{\bar{\varepsilon}} \label{eq:grandmean_full}
\end{align}

\paragraph{Step 2: Law of total variance.}

Let $\mathcal{S}$ denote the judge selection. By the law of total variance:
\begin{equation}
\Var(\bar{X}) = \E\bigl[\Var(\bar{X} \mid \mathcal{S})\bigr] + \Var\bigl(\E[\bar{X} \mid \mathcal{S}]\bigr)
\end{equation}

Conditioning on $\mathcal{S}$:

\emph{Term (I): Conditional variance.} Given $\mathcal{S}$, the bias $\bar{\gamma}$ is a constant. The remaining randomness comes from $\bar{\alpha}$, $\bar{\beta}$, $\bar{\varepsilon}$, which are means of mutually independent zero-mean random variables:
\begin{equation}
\Var(\bar{X} \mid \mathcal{S}) = \frac{\sigma_\alpha^2}{n} + \frac{\sigma_\beta^2}{nm} + \frac{\sigma_\varepsilon^2}{nmK}
\end{equation}
This does not depend on $\mathcal{S}$, so $\E[\Var(\bar{X} \mid \mathcal{S})]$ equals the same expression.

\emph{Term (II): Variance of conditional mean.} Since $\E[\alpha_i] = \E[\beta_{ij}] = \E[\varepsilon_{ij\ell}] = 0$:
\begin{equation}
\E[\bar{X} \mid \mathcal{S}] = \mu_\theta + \bar{\gamma} \implies \Var\bigl(\E[\bar{X} \mid \mathcal{S}]\bigr) = \Var(\bar{\gamma})
\end{equation}

\paragraph{Step 3: Finite population correction.}

$\bar{\gamma}$ is a sample mean of $K$ values drawn without replacement from the centered population $\{\gamma_1,\ldots,\gamma_{\Ktot}\}$.

\begin{lemma}[Finite population correction]
\label{lem:fpc}
Let $a_1, \ldots, a_P$ satisfy $\sum_i a_i = 0$ with $\sigma_a^2 = P^{-1}\sum_i a_i^2$. If $S$ is a simple random sample of size $k$ drawn without replacement:
\begin{equation}
\Var\!\left(\frac{1}{k}\sum_{i \in S} a_i\right) = \frac{\sigma_a^2}{k} \cdot \frac{P - k}{P - 1}
\end{equation}
\end{lemma}

\begin{proof}
Define indicator variables $Z_i \in \{0,1\}$ for $i = 1, \ldots, P$, where $Z_i = 1$ if and only if element $i$ is selected in the sample $S$. The sample sum is $T = \sum_{i=1}^P Z_i a_i$, and the sample mean is $\bar{a}_S = T/k$.

Under simple random sampling without replacement (SRSWOR) of size $k$ from $P$ elements, each element is selected with probability $\Pr(Z_i = 1) = k/P$, so:
\begin{equation}
\Var(Z_i) = \frac{k}{P}\!\left(1 - \frac{k}{P}\right) = \frac{k(P{-}k)}{P^2}
\end{equation}
For $i \neq j$, the joint selection probability is $\Pr(Z_i {=} 1,\, Z_j {=} 1) = \frac{k(k-1)}{P(P-1)}$, giving:
\begin{align}
\Cov(Z_i, Z_j)
  &= \E[Z_i Z_j] - \E[Z_i]\E[Z_j] \nonumber \\
  &= \frac{k(k{-}1)}{P(P{-}1)} - \frac{k^2}{P^2}
  = -\frac{k(P{-}k)}{P^2(P{-}1)}
\end{align}

Since $T = \sum_i Z_i a_i$ and the $a_i$ are constants, expanding the variance gives:
\begin{align}
&\Var(T) \nonumber \\
&= \sum_{i=1}^P a_i^2\,\Var(Z_i) + 2\!\sum_{i<j} a_i a_j\,\Cov(Z_i,Z_j) \nonumber \\
&= \frac{k(P{-}k)}{P^2}\sum_i a_i^2
   - \frac{k(P{-}k)}{P^2(P{-}1)}\cdot 2\!\sum_{i<j} a_i a_j
   \label{eq:varT_expand}
\end{align}

We now simplify each factor. By definition, $\sum_i a_i^2 = P\sigma_a^2$. For the cross-term, the centering constraint $\sum_i a_i = 0$ implies:
\begin{equation}
\begin{split}
0 = \Bigl(\sum_i a_i\Bigr)^{\!2} = \sum_i a_i^2 + 2\!\sum_{i<j} a_i a_j \\
\implies 2\!\sum_{i<j} a_i a_j = -P\sigma_a^2
\end{split}
\end{equation}

Substituting both identities into~\eqref{eq:varT_expand} (the two negatives in the second term yield a positive contribution):
\begin{align}
\Var(T)
  &= \frac{k(P{-}k)}{P^2}\cdot P\sigma_a^2
     + \frac{k(P{-}k)}{P^2(P{-}1)}\cdot P\sigma_a^2 \nonumber \\
  &= \frac{k(P{-}k)\,\sigma_a^2}{P{-}1}
\end{align}

Finally, dividing by $k^2$:
\begin{equation}
\Var(\bar{a}_S) = \frac{\Var(T)}{k^2}
  = \frac{\sigma_a^2}{k}\cdot\frac{P{-}k}{P{-}1} \qedhere
\end{equation}
\end{proof}

The FPC factor equals 1 when $k=1$ and 0 when $k=P$ (offsets cancel exactly). Applying to judge biases ($a_\ell = \gamma_\ell$, $P = \Ktot$, $k = K$):
\begin{equation}
\Var(\bar{\gamma}) = \frac{\sigma_\gamma^2}{K} \cdot \frac{\Ktot - K}{\Ktot - 1}
\end{equation}

Combining Terms (I) and (II) completes the proof.

\section{Estimation of Variance Components}
\label{app:estimation}

Because judges are crossed (not nested), we use a crossed ANOVA.

\paragraph{Residual noise.} The residual mean square estimates $\sigma_\varepsilon^2$:
\begin{align}
\hat{\sigma}_\varepsilon^2 &= \MS_W \nonumber \\
&= \frac{\sum_{i,j,\ell}(X_{ij\ell} - \bar{X}_{ij\cdot} - \bar{X}_{\cdot\cdot\ell} + \bar{X}_{\cdots})^2}{(nm - 1)(K - 1)}
\end{align}

\paragraph{Generation variance.} The between-generation mean square:
\begin{equation}
\MS_G = \frac{K}{n(m-1)}\sum_{i,j}(\bar{X}_{ij\cdot} - \bar{X}_{i\cdot\cdot})^2
\end{equation}
with $\E[\MS_G] = \sigma_\varepsilon^2 + K\sigma_\beta^2$. Judge bias drops out because the same judges evaluate every generation within a scenario.

\paragraph{Scenario variance.} The between-scenario mean square:
\begin{equation}
\MS_S = \frac{mK}{n-1}\sum_i(\bar{X}_{i\cdot\cdot} - \bar{X}_{\cdots})^2
\end{equation}
with $\E[\MS_S] = \sigma_\varepsilon^2 + K\sigma_\beta^2 + mK\,\sigma_\alpha^2$.

\paragraph{Method-of-moments estimators.}
\begin{align}
\hat{\sigma}_\beta^2 &= (\MS_G - \MS_W) / K \\
\hat{\sigma}_\alpha^2 &= (\MS_S - \MS_G) / (mK)
\end{align}

\paragraph{Judge bias.} We estimate $\hat{\gamma}_\ell = \bar{X}_{\cdot\cdot\ell} - \bar{X}_{\cdots}$ and compute:
\begin{equation}
\hat{\sigma}_\gamma^2 = \frac{1}{\Ktot}\sum_\ell \hat{\gamma}_\ell^2 - \frac{\hat{\sigma}_\varepsilon^2}{nm} \cdot \frac{\Ktot - 1}{\Ktot}
\end{equation}

\section{Allocation Strategy Derivations}
\label{app:allocation}

We derive the benchmark score variance for each strategy. The $n$ scenarios are fixed, so $\sigma_\alpha^2/n$ contributes a constant offset and is excluded.

\paragraph{Strategy A: All $\Ktot$ judges per generation.}
Budget $B = m \Ktot$, so $m = B / \Ktot$. Per-scenario mean:
\begin{align}
\bar{X}_i &= \mu_\theta + \alpha_i + \frac{1}{m}\sum_j \beta_{ij} + \underbrace{\frac{1}{\Ktot}\sum_\ell \gamma_\ell}_{=\, 0} \nonumber \\
&\quad + \frac{1}{m\Ktot}\sum_{j,\ell} \varepsilon_{ij\ell}
\end{align}
The judge offsets sum to zero by the centering constraint, so the score equals the panel-mean estimand. The benchmark variance from generation and residual noise:
\begin{equation}
V_A = \frac{1}{n}\!\left(\frac{\sigma_\beta^2}{m} + \frac{\sigma_\varepsilon^2}{m\Ktot}\right) = \frac{\Ktot \sigma_\beta^2 + \sigma_\varepsilon^2}{nB}
\end{equation}

\paragraph{Strategy B: Random single judge per generation.}
Budget $B = m$, each generation scored by $\ell_j \sim \text{Uniform}(\{1, \ldots, \Ktot\})$. Per-scenario mean:
\begin{equation}
\bar{X}_i = \mu_\theta + \alpha_i + \frac{1}{m}\sum_j \bigl(\beta_{ij} + \gamma_{\ell_j} + \varepsilon_{ij\ell_j}\bigr)
\end{equation}
Each term in the sum is independent with variance $\sigma_\beta^2 + \sigma_\gamma^2 + \sigma_\varepsilon^2$:
\begin{equation}
V_B = \frac{\sigma_\beta^2 + \sigma_\gamma^2 + \sigma_\varepsilon^2}{nB}
\end{equation}

\paragraph{Strategy C: Round-robin cycling.}
Budget $B = m$ (multiple of $\Ktot$), generation $j$ scored by judge $j \bmod \Ktot$. Each judge appears $m/\Ktot$ times, so the panel-relative offsets average to zero:
\begin{equation}
\frac{1}{m}\sum_{j=1}^m \gamma_{j \bmod \Ktot} = \frac{1}{\Ktot} \sum_{\ell=1}^{\Ktot} \gamma_\ell = 0
\end{equation}
The residuals for different $j$ are independent with variance $\sigma_\varepsilon^2$:
\begin{equation}
V_C = \frac{\sigma_\beta^2 + \sigma_\varepsilon^2}{nB}
\end{equation}

\paragraph{Pairwise comparisons.} Direct subtraction yields:
\begin{align}
V_A - V_C &= \frac{(\Ktot - 1)\sigma_\beta^2}{nB} \geq 0 \\
V_B - V_C &= \frac{\sigma_\gamma^2}{nB} \geq 0 \\
V_A - V_B &= \frac{(\Ktot - 1)\sigma_\beta^2 - \sigma_\gamma^2}{nB}
\end{align}

\paragraph{Incomplete cycles.}
When $B$ is not a multiple of $\Ktot$, write $B = q\Ktot + r$ with $0 \le r < \Ktot$. Each scenario's $q$ full cycles contribute $q\sum_\ell \gamma_\ell = 0$ by the centering constraint, leaving $r$ remainder calls per scenario. Letting the cycle run on across scenario boundaries makes these remainders cancel against each other as well. The benchmark-level offset is then a single global average over all $nB$ calls, and writing $nB = Q\Ktot + s$ with $0 \le s < \Ktot$, the $Q$ complete passes cancel and exactly $s < \Ktot$ calls are left over:
\begin{equation}
\bar{\gamma} = \frac{1}{nB}\sum_{\ell \in S} \gamma_\ell,
\qquad |\bar{\gamma}| \leq \frac{(\Ktot - 1)\max_\ell |\gamma_\ell|}{nB}
\end{equation}
where $S$ is the set of $s$ judges receiving one extra call. The leftover count $s$ is bounded by the panel size, so $\bar{\gamma}$ is \emph{exactly} zero whenever $\Ktot \mid nB$ and is $O(1/nB)$ otherwise. Since the assignment is deterministic once the scenarios are shuffled, this leftover bias enters the mean squared error as $\bar{\gamma}^2 = O\bigl((1/nB)^2\bigr)$: a factor $O(1/nB)$ below the leading $\sigma_\varepsilon^2/(nB)$ noise term, hence negligible at any realistic $n$, so CyclicJudge retains its advantage over Strategy~B at every budget.

\paragraph{The smallest budget.}
At $B = 1$ every call is a remainder call ($q = 0$, $r = 1$), so each scenario receives a single generation ($m = 1$) scored by a single judge and the cycle runs entirely across scenarios. The bound above applies with $nB = n$: the offsets cancel exactly when $\Ktot \mid n$, and otherwise leave a bias contributing $O\bigl((1/n)^2\bigr)$ to the mean squared error, against the $\sigma_\gamma^2/n$ judge-offset variance that Strategy~B pays at the same budget. Note that a thematically ordered dataset makes judge identity a deterministic function of scenario content, so shuffling becomes essential.

\section{Scenario Coverage}
\label{app:scenario_coverage}

The strategies above fix $n$ and optimize the per-scenario budget. A complementary question is how to split a generation budget $B_{\text{gen}} = nm$ between scenarios and generations per scenario. Under cycling, substituting $n = B_{\text{gen}}/m$ into Eq.~\ref{eq:strat_cycle} gives:
\begin{equation}
\label{eq:scenario_gen}
V_C(B_{\text{gen}}, m) = \frac{m\,\sigma_\alpha^2 + \sigma_\beta^2 + \sigma_\varepsilon^2}{B_{\text{gen}}}
\end{equation}
This is linear in $m$ with positive slope $\sigma_\alpha^2/B_{\text{gen}}$, so variance is minimized at $m = 1$, maximizing the number of distinct scenarios up to the dataset size $N$. The same monotonicity holds for Strategies~A and~B. Once the scenario pool is exhausted ($n = N$), excess budget should go to additional generations.

\section{Exact Bootstrap Predictions}
\label{app:pool_pred}

The allocation formulas in \S\ref{sec:allocation} assume infinite-population generation draws. When validating via bootstrap subsampling from a finite pool of $m_{\max}$ observed generations, we compute exact predictions using the empirical pool variances. For all three strategies, the subsampling observations within each scenario are i.i.d.\ draws from the relevant pool, so $V(B) = C/(nB)$ where $C$ is the pool variance.

\paragraph{Strategy A.} Each draw is a generation-level mean $\bar{X}_{ig\cdot}$, with $m = B/\Ktot$ draws per scenario. The pool variance is:
\begin{equation}
C_A = \Ktot \cdot \frac{1}{n}\sum_i \frac{1}{m_{\max}}\sum_g (\bar{X}_{ig\cdot} - \bar{X}_{i\cdot\cdot})^2
\end{equation}

\paragraph{Strategy B.} Each draw is a single cell $X_{ig\ell}$ with generation and judge chosen uniformly at random, with $m = B$ draws per scenario:
\begin{equation}
C_B = \frac{1}{n}\sum_i \frac{1}{m_{\max}\Ktot}\sum_{g,\ell} (X_{ig\ell} - \bar{X}_{i\cdot\cdot})^2
\end{equation}

\paragraph{Strategy C.} Each judge $\ell$ gets $B/\Ktot$ independent draws from its column; the per-scenario mean averages these $\Ktot$ independent sub-means. The relevant pool variance is the within-column generation variance:
\begin{equation}
C_C = \frac{1}{n\Ktot}\sum_{i,\ell} \frac{1}{m_{\max}}\sum_g (X_{ig\ell} - \bar{X}_{i\cdot\ell})^2
\end{equation}

These empirical constants converge to the population formulas (Eqs.~\ref{eq:strat_all}--\ref{eq:strat_cycle}) as $m_{\max} \to \infty$, but for finite $m_{\max}$ they account for the bias in plug-in variance estimates (the $\nicefrac{(m_{\max}{-}1)}{m_{\max}}$ correction) and the fact that the crossed ANOVA residual uses global judge centering while $C_C$ uses within-scenario-judge centering.

\section{Ranking Stability}
\label{app:ranking}

\paragraph{Method.}
Models are compared on the intersection of scenarios present for all five evaluated models, so the leaderboard is well defined ($n = 80$ on MT-Bench, $n = 50$ on MindEval). The \emph{reference panel ranking} is the leaderboard induced by the full-data grand mean of each model, averaging over all generations and all $\Ktot$ judges. It orders the models Gemini $>$ GPT $>$ Claude $>$ Llama $>$ Qwen on MT-Bench, and Claude $>$ GPT $>$ Gemini $>$ Llama $>$ Qwen on MindEval, matching the $\hat{\mu}_\theta$ estimates of Table~\ref{tab:components}. For each strategy and per-scenario budget $B$ we draw $5{,}000$ subsampling repetitions using the same subsampler as the variance experiments (\S\ref{sec:experiments}), obtaining a $5{,}000 \times 5$ matrix of benchmark scores, and compute three statistics per repetition:

\begin{itemize}[leftmargin=*,itemsep=2pt]
\item \textbf{Spearman $\rho$}: the Pearson correlation between the tie-corrected ranks of the repetition's scores and the reference ranks.
\item \textbf{Kendall $\tau$-b}: computed over all $\binom{5}{2} = 10$ model pairs as the concordant-minus-discordant count normalized by $\sqrt{(10 - t_{\text{ref}})(10 - t_{\text{rep}})}$, where $t_{\text{ref}}$ and $t_{\text{rep}}$ count pairs tied in the reference and in the repetition respectively. The $\tau$-b form is used because subsampled scores can tie.
\item \textbf{Top-1 flip}: the indicator that the repetition's top-1 model differs from the reference.
\end{itemize}

The \emph{single judge} strategy is the budget-matched version of picking one judge in practice: each repetition draws one judge uniformly at random and uses it to score all generations. Generations are drawn independently per model, as for the other strategies. Separately, Table~\ref{tab:fixed_judge} reports the \emph{full-data} per-judge leaderboards corresponding to Table~\ref{tab:mean_scores}, where each judge scores every model on all scenarios and generations. These are the $B \to \infty$ limit of the single-judge strategy, in which generation and residual noise have averaged away and only the judge's systematic offset remains; they therefore understate how unstable a single judge is at any finite budget.

\begin{table}[ht]
\centering
\caption{Per-judge leaderboards from the full crossed data (each judge scores every model on all scenarios and generations), with rank correlation to the reference panel ranking. Rows marked $^\dagger$ disagree with the panel top-1 model.}
\label{tab:fixed_judge}
\smallskip
{\small (a) MT-Bench --- panel top-1: Gemini 3 Flash}
\vspace{2pt}

\resizebox{\columnwidth}{!}{%
\begin{tabular}{@{}lccc@{}}
\toprule
\textbf{Judge} & \textbf{Top-1 model} & $\rho$ & $\tau$-b \\
\midrule
Qwen 2.5 7B       & Gemini 3 Flash    & 0.300 & 0.200 \\
Llama 3.3 70B     & Claude Sonnet 4.6$^\dagger$ & 0.700 & 0.600 \\
GPT-5.2           & GPT-5.2$^\dagger$ & 0.900 & 0.800 \\
Gemini 3 Flash    & Gemini 3 Flash    & 1.000 & 1.000 \\
Claude Sonnet 4.6 & Gemini 3 Flash    & 0.900 & 0.800 \\
\midrule
\emph{Mean} & 3/5 distinct & 0.760 & 0.680 \\
\bottomrule
\end{tabular}}

\medskip
{\small (b) MindEval --- panel top-1: Claude Sonnet 4.6}
\vspace{2pt}

\resizebox{\columnwidth}{!}{%
\begin{tabular}{@{}lccc@{}}
\toprule
\textbf{Judge} & \textbf{Top-1 model} & $\rho$ & $\tau$-b \\
\midrule
Qwen 2.5 7B       & GPT-5.2$^\dagger$ & 0.700 & 0.600 \\
Llama 3.3 70B     & Claude Sonnet 4.6 & 1.000 & 1.000 \\
GPT-5.2           & Claude Sonnet 4.6 & 1.000 & 1.000 \\
Gemini 3 Flash    & Claude Sonnet 4.6 & 1.000 & 1.000 \\
Claude Sonnet 4.6 & Claude Sonnet 4.6 & 1.000 & 1.000 \\
\midrule
\emph{Mean} & 2/5 distinct & 0.940 & 0.920 \\
\bottomrule
\end{tabular}}
\end{table}

\paragraph{Results.}
Table~\ref{tab:fixed_judge} reports the fixed single-judge leaderboards. On MT-Bench three distinct models take the top slot depending on which judge is used, and two of the five judges disagree with the panel winner. MindEval is markedly more stable: only Qwen 2.5 7B, the weakest judge, produces a different top-1 model.

Sweeping the budget confirms that the ordering of Table~\ref{tab:ranking} is not specific to $B{=}5$. On MT-Bench, CyclicJudge is best on all three metrics at every budget we tested ($B = 5, 10, \ldots, 50$), and its advantage in top-1 flip frequency widens in relative terms as $B$ grows: at $B{=}50$ its flip rate is $0.002$, against $0.067$ for random single-judge and $0.025$ for all-judges. The single-judge flip rate, by contrast, does not improve with budget at all, staying between $0.49$ and $0.51$ across the entire sweep: extra generations average away generation and residual noise but leave the judge's own offset untouched, which is the practical form of the $V_\gamma$ term of Eq.~\ref{eq:variance}. On MindEval the three multi-judge allocations attain a near-perfect ordering from $B{=}5$ upward with no top-1 flips at any budget, so that benchmark cannot discriminate between them; CyclicJudge is best or tied throughout. The single judge again does not improve with budget, holding a flip rate near $0.20$ across the sweep. The gap between allocations is therefore benchmark-dependent.

\section{ANOVA \texorpdfstring{$F$}{F}-Test for Judge Effects}
\label{app:anova}

We test for a systematic judge effect using a two-way crossed ANOVA that treats each scenario--generation pair as a subject crossed with $K$ judges. The judge sum of squares $\text{SS}_J = N \sum_{\ell=1}^{K}(\bar{X}_{\cdot\cdot\ell} - \bar{X}_{\cdots})^2$ (where $N = nm$) is compared to the residual $\text{SS}_R = \text{SS}_T - \text{SS}_J - \text{SS}_S$, yielding $F = \text{MS}_J / \text{MS}_R$ with $K{-}1$ and $(N{-}1)(K{-}1)$ degrees of freedom. The $p$-value is the survival function of the $F$-distribution, computed via \texttt{scipy.stats.f.sf}. Table~\ref{tab:anova} reports results for both benchmarks; in all cases, the null hypothesis of no judge effect is rejected at $p < 0.001$.

\begin{table}[ht]
\centering
\caption{ANOVA \texorpdfstring{$F$}{F}-statistic for the judge main effect. All \texorpdfstring{$p < 0.001$}{p < 0.001}.}
\label{tab:anova}
\begin{tabular}{@{}lcc@{}}
\toprule
& \textbf{MT-Bench} & \textbf{MindEval} \\
\textbf{Model} & $F(4,\,3196)$ & $F(4,\,996)$ \\
\midrule
Qwen 2.5 7B       & 638.2 & 1029.5 \\
Llama 3.3 70B      & 445.8 & 1280.1 \\
GPT-5.2            & 601.7 & 200.1 \\
Gemini 3 Flash     & 659.2 & 458.0 \\
Claude Sonnet 4.6  & 514.9 & 788.9 \\
\bottomrule
\end{tabular}
\end{table}

\section{Omitted Interactions}
\label{app:interactions}

A full factorial design has $2^3 = 8$ terms: a grand mean, three main effects, three two-way interactions, and one three-way interaction. Our model (Eq.~\ref{eq:model}) retains the grand mean and three main effects, absorbing all interaction terms into the residual $\varepsilon_{ij\ell}$. This is why CyclicJudge should be understood as canceling each judge's \emph{average} additive offset rather than as recovering the full panel mean exactly: whatever interaction structure exists is pushed into $\sigma_\varepsilon^2$, where cycling reduces it at the $O(1/B)$ noise rate instead of removing it. We enumerate each omitted interaction below.

\paragraph{Scenario--generation.} Because generations are nested within scenarios, there is no separable scenario--generation interaction: any such effect is already confounded with the generation main effect $\beta_{ij}$, which captures all within-scenario generation variability by construction.

\paragraph{Scenario--judge.} This interaction captures the extent to which a judge's bias varies across scenarios (e.g., a judge being lenient on reasoning tasks but strict on creative ones). It is in principle estimable, but explicitly modeling it would add $n(\Ktot{-}1)$ parameters and complicate the closed-form variance decomposition without changing the allocation analysis: the interaction enters only through the per-cell residual variance, which our empirical predictions (\S\ref{sec:experiments}) already capture without parametric assumptions on its structure.

\paragraph{Generation--judge.} This interaction reflects judge-specific reactions to particular generation artifacts (e.g., one judge penalizing verbose outputs more than another). Since generations are independent stochastic draws, this term is exchangeable across $j$ and acts as additional i.i.d.\ noise on each cell, making it naturally absorbed into $\sigma_\varepsilon^2$.

\section{Experimental Configuration}
\label{app:hyperparameters}

For all judge calls we follow the default decoding hyperparameters specified by the MT-Bench and MindEval reference implementations, summarized in Table~\ref{tab:hyperparameters}. For models in our judge pool that fall outside the original API tables of either reference implementation (e.g., Gemini 3 Flash and Claude Sonnet 4.6 for MT-Bench), we apply the closest equivalent setting supported by the provider.

\begin{table}[ht]
\centering
\caption{Default judge decoding hyperparameters used by each benchmark's reference implementation.}
\label{tab:hyperparameters}
\smallskip
\begin{tabular}{@{}ll@{}}
\toprule
\textbf{Parameter} & \textbf{Default} \\
\midrule
\multicolumn{2}{@{}l}{\emph{MT-Bench}} \\
\quad temperature & $0$ \\
\quad max\_tokens (OpenAI) & $2048$ \\
\quad max\_tokens (Anthropic) & $1024$ \\
\midrule
\multicolumn{2}{@{}l}{\emph{MindEval}} \\
\quad judge\_template\_version & \texttt{v0\_1} \\
\quad max\_completion\_tokens & $16{,}000$ \\
\quad reasoning\_effort & \texttt{high} \\
\bottomrule
\end{tabular}
\end{table}

\paragraph{Effect of hyperparameters.}
Decoding hyperparameters such as temperature and reasoning effort primarily control the within-judge stochasticity of a single call. In the model of Eq.~\ref{eq:model}, this stochasticity is absorbed into the residual $\varepsilon_{ij\ell}$, contributing to $\sigma_\varepsilon^2$. The systematic offset $\gamma_\ell$, by contrast, is a property of the judge model itself, not of how its outputs are sampled. Increasing temperature would inflate $\sigma_\varepsilon^2$ but, to first order, leave $\gamma_\ell$ and $\sigma_\gamma^2$ unchanged. At the default operating point, the judge-offset contribution exceeds the residual contribution by roughly two orders of magnitude on both benchmarks (Table~\ref{tab:components}). Reasonable perturbations of the decoding hyperparameters can only shift $\sigma_\varepsilon^2$ within a window that is much smaller than $\sigma_\gamma^2$, so the qualitative ranking of allocation strategies in \S\ref{sec:allocation} and Figure~\ref{fig:strategy} is robust to this choice.

\paragraph{Licenses and intended use.}
MT-Bench is released under CC BY 4.0, and MindEval is publicly released for research purposes; both benchmarks are intended for model evaluation.

\end{document}